\documentclass{article}
\usepackage[final]{colm2026_conference}

\usepackage[utf8]{inputenc} 
\usepackage[T1]{fontenc}    
\usepackage{hyperref}       
\usepackage{url}            
\usepackage{booktabs}       
\usepackage{amsfonts}       
\usepackage{wrapfig,lipsum}
\usepackage{amsthm,mathrsfs,amsmath,amssymb}
\usepackage{caption}
\usepackage{subcaption} 
\usepackage{graphicx,graphics}
\usepackage{svg}
\usepackage{relsize}
\usepackage[noend]{algpseudocode}  
\usepackage{algorithm}
\usepackage{algorithmicx}
\usepackage{bbm}
\usepackage{bm}
\usepackage{tikz}
\usetikzlibrary{3d, calc}
\usepackage{svg}
\usepackage{xspace}
\graphicspath{{figures/}}
\usepackage[capitalize,noabbrev]{cleveref}
\usepackage{xcolor}
\usepackage[table,dvipsnames]{xcolor}
\usepackage{paralist} 
\usepackage{multirow, makecell, colortbl, hhline}
\usepackage{dirtytalk}

\newcommand{\ourmethod}{\texttt{ADMM-Q}\xspace}

\definecolor{aurometalsaurus}{rgb}{0.43, 0.5, 0.5}
\definecolor{britishracinggreen}{rgb}{0.0, 0.26, 0.15}
\definecolor{burntumber}{rgb}{0.54, 0.2, 0.14}
\definecolor{cobalt}{rgb}{0.0, 0.28, 0.67}
\definecolor{bulgarianrose}{rgb}{0.28, 0.02, 0.03}
\definecolor{ceruleanblue}{rgb}{0.16, 0.32, 0.75}
\definecolor{darkgreen}{RGB}{0,128,0}
\definecolor{darkred}{rgb}{0.74, 0.2, 0.24}

\newcommand{\argmin}{\operatornamewithlimits{\arg\min}}

\newtheorem{thm}{Theorem}

\usepackage[dvipsnames]{xcolor}
\usepackage{tikz}
\usetikzlibrary{shapes.geometric, arrows}
\usetikzlibrary{positioning}
\usetikzlibrary {arrows.meta}
\usepackage{mathtools}
\usetikzlibrary{decorations.pathreplacing,calligraphy}
\usepackage{enumitem}
\usepackage{siunitx}
\newcommand{\crnew}[1]{#1}

\title{
  ADMM-Q: An Improved Hessian-based Weight Quantizer for Post-Training Quantization of Large Language Models
}

\author{
\makebox[\textwidth][c]{%
\begin{tabular}{c}
Ryan Lucas\thanks{Equal contribution.},\;
Mehdi Makni\footnotemark[1],\;
Xiang Meng\footnotemark[1],\;
Adam Deng,\;
Rahul Mazumder \\[0.35em]
Operations Research Center \\
Massachusetts Institute of Technology
\end{tabular}%
}
}

\begin{document}
\pagestyle{fancy}
\fancyhead{}
\lhead{Published as a conference paper at COLM 2026}
\renewcommand{\headrulewidth}{1.5pt}
\fancyfoot{}
\cfoot{\thepage}
\maketitle
\begin{abstract}

Quantization is an effective strategy to reduce the storage and computation footprint of large language models (LLMs). Post-training quantization (PTQ) is a leading approach for compressing LLMs. Popular weight quantization procedures, including GPTQ and RTN, suffer in model utility, especially at aggressive quantization levels (sub-4-bit). We propose \ourmethod, a novel weight quantization algorithm that considers the layer-wise quantization problem. Our algorithm is based on a combinatorial variant of the Alternating Direction Method of Multipliers (ADMM). 
Our operator-splitting procedure updates weights continuously to minimize the layer-wise reconstruction error, while gradually enforcing the quantization constraints with convergence guarantees.
We propose additional algorithmic enhancements (e.g., penalty scheduling, preconditioning, and a local search post-processing step) to make \ourmethod efficient at LLM scale. \ourmethod is modular and can be used as a drop-in replacement for any weight quantizer within existing quantization pipelines: \ourmethod is fully composable with existing techniques including range clipping, learned or random rotations, and activation scaling. Using \ourmethod in place of GPTQ on Qwen3-8B, we decrease WikiText-2 perplexity in: (i) the W3A16 weight-only setting (12.85 $\rightarrow$ 10.06); (ii) the W4A8 SmoothQuant procedure (9.29 $\rightarrow$ 8.68); and (iii) the W2A4KV4 SpinQuant procedure (66.11 $\rightarrow$ 19.42). 

\end{abstract}

\vspace{-1em}
\section{Introduction}
\label{sec:introduction}
\vspace{-0.5em}
Large Language Models (LLMs) have demonstrated exceptional performance across diverse tasks including complex reasoning \citep{xu2025towards}, text generation \citep{achiam2023gpt}, mathematical problem-solving \citep{deepmindmath}, and code synthesis \citep{roziere2023code}. However, state-of-the-art LLMs~\citep{achiam2023gpt,dubey2024llama,team2023gemini} with billions of parameters face substantial deployment challenges due to their computational and memory requirements. These constraints limit real-time applications and deployment on resource-constrained devices, making model compression essential for increasing LLM accessibility while preserving accuracy.

Quantization~\citep{han2015deep,frantar2022gptq,lin2024awq} has emerged as a practical compression approach for LLMs, reducing weight precision from 16-bit floating point to 4 or fewer bits with minimal accuracy degradation. Weight-only quantization (e.g., W4A16) compresses model storage and memory bandwidth while keeping activations in full precision; weight-and-activation quantization (e.g., W4A4) additionally quantizes activations for faster matrix multiplications \cite{xiao2023smoothquant, liu2024spinquant}. In practice, most post-training quantization (PTQ) pipelines process the model layer by layer, minimizing the reconstruction error $\|\mathbf{X}\widehat{\mathbf{W}} - \mathbf{X}\mathbf{W}\|_F^2$ between the original and quantized layer outputs on a small calibration set. The two standard weight quantizers used in this step are round-to-nearest (RTN) and GPTQ~\citep{frantar2022gptq}.

RTN independently rounds each weight to its nearest quantization grid point post clipping value selection, ignoring inter-weight dependencies entirely. While computationally efficient, this approach is suboptimal because it fails to account for the structure of the reconstruction error landscape---weights with large Hessian entries receive no special treatment. GPTQ addresses this by incorporating Hessian information, but it processes weights greedily in a column-by-column fashion: after quantizing each column, the remaining weights are updated to compensate for the incurred error. This sequential procedure is sensitive to processing order and accumulates errors across columns.

We propose \ourmethod, a principled weight quantizer that formulates layer-wise quantization as a constrained optimization problem and solves it via the Alternating Direction Method of Multipliers (ADMM)~\citep{boyd2011distributed}. Unlike GPTQ's greedy column-wise approach, \ourmethod jointly optimizes all weights simultaneously by alternating between a continuous Hessian-weighted update and a discrete quantization projection.
\ourmethod is designed as a \emph{drop-in replacement} for RTN or GPTQ within any existing PTQ pipeline, including weight-only quantization pipelines or those that employ rotation-based~\citep{liu2024spinquant} or scaling-based~\citep{xiao2023smoothquant,lin2024awq} transformations in addition to learned clipping heuristics (e.g., MSE clipping) for weight-and-activation quantization. Quantization to low bit-widths poses a fundamentally hard discrete optimization problem, and as our experiments demonstrate, the choice of solver has a substantial impact on the resulting model utility-compression trade-off.

\vspace{-0.8em}
\begin{figure}[!h]
\centering
\begin{tikzpicture}[scale=0.65, transform shape, every node/.style={inner sep=0pt, outer sep=0pt, scale=1.25}]

\definecolor{qA}{RGB}{230,241,251}
\definecolor{qB}{RGB}{181,212,244}
\definecolor{qC}{RGB}{133,183,235}
\definecolor{qD}{RGB}{55,138,221}

\definecolor{cardbg}{RGB}{248,248,248}
\definecolor{panelbg}{RGB}{252,252,252}
\definecolor{lightgrid}{RGB}{198,198,198}
\definecolor{accent}{RGB}{210,75,50}

\draw[rounded corners=5pt, fill=cardbg, draw=black!50, line width=0.85pt]
  (0.2,0.2) rectangle (5.0,-5.0);

\node[anchor=west] at (0.35,-0.40) {\small \textbf{Layerwise reconstruction}};
\draw[black!22, line width=0.5pt] (0.35,-0.7) -- (4.8,-0.7);

\node[align=center] at (2.55,-2.95)
  {$\displaystyle
    \min_{{\mathbf W}\in\mathcal Q}
    \ \|\mathbf Y-\mathbf X{\mathbf W}\|_F^2
  $};

\node[align=center] at (2.55,-1.6)
{\scriptsize Quantize the layer weights to\\[2pt]
   \scriptsize points on a quantization grid};

\node[align=center] at (2.55,-4.25)
{\scriptsize $\mathbf X$: input activations\\[3pt]
   \scriptsize $\mathbf Y$: output activations};

\begin{scope}[shift={(0,1.2)}]
\draw[rounded corners=3pt, fill=cardbg, draw=black!60]
  (0.85,-6.60) rectangle (4.55,-7.90);
\node[anchor=west] at (1.0,-6.90) {\scriptsize Quantization grid (2-bits)};

\draw[black!55, line width=0.5pt, ->] (1.55,-7.45) -- (3.95,-7.45);

\fill[qA,draw=black!55] (1.75,-7.3) circle (1.8pt);
\node[anchor=north] at (1.75,-7.60) {\scriptsize $q_{00}$};

\fill[qB,draw=black!55] (2.35,-7.3) circle (1.8pt);
\node[anchor=north] at (2.35,-7.60) {\scriptsize $q_{01}$};

\fill[qC,draw=black!55] (2.95,-7.3) circle (1.8pt);
\node[anchor=north] at (2.95,-7.60) {\scriptsize $q_{10}$};

\fill[qD,draw=black!55] (3.55,-7.3) circle (1.8pt);
\node[anchor=north] at (3.55,-7.60) {\scriptsize $q_{11}$};
\end{scope}

\draw[line width=0.95pt, double distance=1pt,
      arrows={-Latex[length=0pt 3 0]}] (5.15,-2.85) -- (6.30,-2.85);
\node at (5.7,-2.5) {\scriptsize ADMM};

\begin{scope}[shift={(-0.3,-0.25)}]
\draw[rounded corners=4pt, fill=panelbg, draw=black!22]
  (6.75,0.80) rectangle (11.4,-6.65);

\fill[qA] (7.05,0.45) rectangle (8.05,-0.55);
\fill[qD] (8.05,0.45) rectangle (9.05,-0.55);
\fill[qB] (9.05,0.45) rectangle (10.05,-0.55);
\fill[qD] (10.05,0.45) rectangle (11.05,-0.55);

\fill[qB] (7.05,-0.55) rectangle (8.05,-1.55);
\fill[qA] (8.05,-0.55) rectangle (9.05,-1.55);
\fill[qD] (9.05,-0.55) rectangle (10.05,-1.55);
\fill[qA] (10.05,-0.55) rectangle (11.05,-1.55);

\fill[qA] (7.05,-1.55) rectangle (8.05,-2.55);
\fill[qB] (8.05,-1.55) rectangle (9.05,-2.55);
\fill[qA] (9.05,-1.55) rectangle (10.05,-2.55);
\fill[qC] (10.05,-1.55) rectangle (11.05,-2.55);

\fill[qD] (7.05,-2.55) rectangle (8.05,-3.55);
\fill[qA] (8.05,-2.55) rectangle (9.05,-3.55);
\fill[qC] (9.05,-2.55) rectangle (10.05,-3.55);
\fill[qD] (10.05,-2.55) rectangle (11.05,-3.55);

\fill[qC] (7.05,-3.55) rectangle (8.05,-4.55);
\fill[qD] (8.05,-3.55) rectangle (9.05,-4.55);
\fill[qA] (9.05,-3.55) rectangle (10.05,-4.55);
\fill[qB] (10.05,-3.55) rectangle (11.05,-4.55);

\fill[qB] (7.05,-4.55) rectangle (8.05,-5.55);
\fill[qC] (8.05,-4.55) rectangle (9.05,-5.55);
\fill[qD] (9.05,-4.55) rectangle (10.05,-5.55);
\fill[qA] (10.05,-4.55) rectangle (11.05,-5.55);

\draw[line width=0.75pt] (7.05,0.45) rectangle (11.05,-5.55);
\draw[line width=0.75pt] (8.05,0.45) -- (8.05,-5.55);
\draw[line width=0.75pt] (9.05,0.45) -- (9.05,-5.55);
\draw[line width=0.75pt] (10.05,0.45) -- (10.05,-5.55);
\draw[line width=0.75pt] (7.05,-0.55) -- (11.05,-0.55);
\draw[line width=0.75pt] (7.05,-1.55) -- (11.05,-1.55);
\draw[line width=0.75pt] (7.05,-2.55) -- (11.05,-2.55);
\draw[line width=0.75pt] (7.05,-3.55) -- (11.05,-3.55);
\draw[line width=0.75pt] (7.05,-4.55) -- (11.05,-4.55);

\node[align=center] at (9.05,-6.15)
  {\scriptsize Quantized $\mathbf{W}$ from ADMM};
\end{scope}
\draw[line width=0.95pt, double distance=1pt,
      arrows={-Latex[length=0pt 3 0]}] (11.20,-2.85) -- (12.35,-2.85);
\node[align=center] at (11.7,-2.2)
  {\scriptsize Local\\[-1pt] \scriptsize search};
  
\begin{scope}[shift={(-0.6,-0.25)}]
\draw[rounded corners=4pt, fill=panelbg, draw=black!22]
  (13.05,0.80) rectangle (20.75,-6.75);

\fill[black!3] (13.45,0.45) rectangle (17.45,-5.55);

\fill[qA,opacity=0.42] (13.45,0.45) rectangle (14.45,-0.55);
\fill[qD,opacity=0.42] (14.45,0.45) rectangle (15.45,-0.55);
\fill[qB,opacity=0.42] (15.45,0.45) rectangle (16.45,-0.55);
\fill[qD,opacity=0.42] (16.45,0.45) rectangle (17.45,-0.55);

\fill[qB,opacity=0.42] (13.45,-0.55) rectangle (14.45,-1.55);
\fill[qA,opacity=0.42] (14.45,-0.55) rectangle (15.45,-1.55);
\fill[qD,opacity=0.42] (15.45,-0.55) rectangle (16.45,-1.55);
\fill[qA,opacity=0.42] (16.45,-0.55) rectangle (17.45,-1.55);

\fill[qA,opacity=0.42] (13.45,-1.55) rectangle (14.45,-2.55);
\fill[qB,opacity=0.42] (14.45,-1.55) rectangle (15.45,-2.55);
\fill[qA,opacity=0.42] (15.45,-1.55) rectangle (16.45,-2.55);
\fill[qC,opacity=0.42] (16.45,-1.55) rectangle (17.45,-2.55);

\fill[qD,opacity=0.42] (13.45,-2.55) rectangle (14.45,-3.55);
\fill[qA,opacity=0.42] (14.45,-2.55) rectangle (15.45,-3.55);
\fill[qC,opacity=0.42] (15.45,-2.55) rectangle (16.45,-3.55);
\fill[qD,opacity=0.42] (16.45,-2.55) rectangle (17.45,-3.55);

\fill[qC,opacity=0.42] (13.45,-3.55) rectangle (14.45,-4.55);
\fill[qD,opacity=0.42] (14.45,-3.55) rectangle (15.45,-4.55);
\fill[qA,opacity=0.42] (15.45,-3.55) rectangle (16.45,-4.55);
\fill[qB,opacity=0.42] (16.45,-3.55) rectangle (17.45,-4.55);

\fill[qB,opacity=0.42] (13.45,-4.55) rectangle (14.45,-5.55);
\fill[qC,opacity=0.42] (14.45,-4.55) rectangle (15.45,-5.55);
\fill[qD,opacity=0.42] (15.45,-4.55) rectangle (16.45,-5.55);
\fill[qA,opacity=0.42] (16.45,-4.55) rectangle (17.45,-5.55);

\draw[draw=lightgrid, line width=0.75pt] (13.45,0.45) rectangle (17.45,-5.55);
\draw[draw=lightgrid, line width=0.75pt] (14.45,0.45) -- (14.45,-5.55);
\draw[draw=lightgrid, line width=0.75pt] (15.45,0.45) -- (15.45,-5.55);
\draw[draw=lightgrid, line width=0.75pt] (16.45,0.45) -- (16.45,-5.55);
\draw[draw=lightgrid, line width=0.75pt] (13.45,-0.55) -- (17.45,-0.55);
\draw[draw=lightgrid, line width=0.75pt] (13.45,-1.55) -- (17.45,-1.55);
\draw[draw=lightgrid, line width=0.75pt] (13.45,-2.55) -- (17.45,-2.55);
\draw[draw=lightgrid, line width=0.75pt] (13.45,-3.55) -- (17.45,-3.55);
\draw[draw=lightgrid, line width=0.75pt] (13.45,-4.55) -- (17.45,-4.55);

\node[align=center] at (15.45,-6.15)
  {\scriptsize Final quantized $\mathbf{W}$\\[-1pt]\scriptsize after local search};

\fill[qB] (14.45,-2.55) rectangle (15.45,-3.55); 
\fill[qD] (14.45,-4.55) rectangle (15.45,-5.55); 
\draw[draw=accent, line width=1.0pt] (14.45,-2.55) rectangle (15.45,-3.55);
\draw[draw=accent, line width=1.0pt] (14.45,-4.55) rectangle (15.45,-5.55);

\fill[qC] (16.45,0.45) rectangle (17.45,-0.55);  
\fill[qB] (16.45,-0.55) rectangle (17.45,-1.55); 
\draw[draw=accent, line width=1.0pt] (16.45,0.45) rectangle (17.45,-0.55);
\draw[draw=accent, line width=1.0pt] (16.45,-0.55) rectangle (17.45,-1.55);

\node[anchor=west] at (17.8,-5.05) {\scriptsize Discrete local moves};

\draw[rounded corners=3pt, fill=white, draw=black!28]
  (18.05,0.2) rectangle (20.35,-2.0);

\fill[qD] (18.2,-0.05) rectangle (18.8,-0.65);
\draw[black!35] (18.2,-0.05) rectangle (18.8,-0.65);
\fill[qC] (19.6,-0.05) rectangle (20.2,-0.65);
\draw[black!35] (19.6,-0.05) rectangle (20.2,-0.65);

\draw[-Latex, line width=0.35pt] (18.85,-0.35) -- (19.55,-0.35);

\node at (18.5,-0.8) {\tiny $q_{11}$};
\node at (19.9,-0.8) {\tiny $q_{10}$};

\fill[qA] (18.2,-1.05) rectangle (18.8,-1.65);
\draw[black!35] (18.2,-1.05) rectangle (18.8,-1.65);
\fill[qB] (19.6,-1.05) rectangle (20.2,-1.65);
\draw[black!35] (19.6,-1.05) rectangle (20.2,-1.65);

\draw[-Latex, line width=0.35pt] (18.85,-1.35) -- (19.55,-1.35);

\node at (18.5,-1.8) {\tiny $q_{00}$};
\node at (19.9,-1.8) {\tiny $q_{01}$};

\begin{scope}[shift={(0,-0.6)}]

\draw[rounded corners=3pt, fill=white, draw=black!28]
  (18.05,-1.8) rectangle (20.35,-4.0);

\fill[qA] (18.2,-2.05) rectangle (18.8,-2.65);
\draw[black!35] (18.2,-2.05) rectangle (18.8,-2.65);
\fill[qB] (19.6,-2.05) rectangle (20.2,-2.65);
\draw[black!35] (19.6,-2.05) rectangle (20.2,-2.65);

\draw[-Latex, line width=0.35pt] (18.85,-2.35) -- (19.55,-2.35);

\node at (18.5,-2.8) {\tiny $q_{00}$};
\node at (19.9,-2.8) {\tiny $q_{01}$};

\fill[qC] (18.2,-3.05) rectangle (18.8,-3.65);
\draw[black!35] (18.2,-3.05) rectangle (18.8,-3.65);
\fill[qD] (19.6,-3.05) rectangle (20.2,-3.65);
\draw[black!35] (19.6,-3.05) rectangle (20.2,-3.65);

\draw[-Latex, line width=0.35pt] (18.85,-3.35) -- (19.55,-3.35);

\node at (18.5,-3.8) {\tiny $q_{10}$};
\node at (19.9,-3.8) {\tiny $q_{11}$};
\end{scope}

\draw[accent!90, line width=0.45pt, -Latex]
  (15.45,-4.05) -- (18.10,-3.75);
\draw[accent!90, line width=0.45pt, -Latex]
  (17.45,-0.55) -- (18.1,-0.9);
\end{scope}

\end{tikzpicture}
\caption{\small{Overview of the proposed \ourmethod~algorithm. (\textbf{Left}) The layerwise quantization problem with a reconstruction objective; the goal is to approximate the full-precision weight matrix $\mathbf{W}$ using quantized weights (Section \ref{sec:problem-formulation}). (\textbf{Middle}) ADMM with diagonal scaling and $\rho$-update scheme (Algorithm \ref{alg:admm}) to obtain a high-quality quantized weight matrix (Section \ref{sec:admm-algorithm}). (\textbf{Right}) Starting from the ADMM solution, a local search procedure is applied to further improve the quantized weights and reduce the reconstruction error (Section \ref{sec:local-search}).}}
\label{fig:flowchart}
\end{figure}

Our contributions can be summarized as follows:
\begin{itemize}[leftmargin=*,itemsep=2pt,topsep=2pt]
\item We formulate layer-wise weight quantization as a Hessian-weighted constrained optimization problem and propose \ourmethod, an ADMM-based solver with diagonal scaling for numerical stability, an adaptive penalty schedule, and a local-search refinement step for improved discrete solutions (\cref{sec:problem-formulation}). Further, we extend the convergence proof of ADMM established under pruning constraints \cite{meng2024alps} to our setting involving quantization constraints (\cref{thm:admmq}).

\item In weight-only quantization, \ourmethod consistently improves over GPTQ across the Qwen3 model family in perplexity and zero-shot accuracy, with especially large gains in aggressive low-bit regimes such as W3 (\cref{tab:qwen3-architecture-sweep}).

\item When integrated into existing weight-and-activation quantization pipelines such as SpinQuant and SmoothQuant, \ourmethod continues to improve over GPTQ (\cref{tab:sq-llama-results,tab:qwen3-results}). Moreover, when deployed in vLLM it preserves identical inference throughput and memory benefits as GPTQ (for instance, $\approx$1.5x speedup at W8A8 on Qwen-32B; \cref{tab:qwen_latency_speedup_gptq_admmq}).
    
\end{itemize}

Our code is available at: \url{https://github.com/mazumder-lab/admm-q}.

\vspace{-1em}
\section{Related Work}
\label{sec:related-work}
\vspace{-0.5em}
\noindent\textbf{Weight-only quantizers of LLMs.}~~~PTQ compresses pre-trained models without retraining by reducing weight and/or activation precision. While activations are quantized on-the-fly during inference, weights are fixed during deployment and can be compressed offline using a specialized optimization procedure. A popular weight quantizer is GPTQ~\citep{frantar2022gptq}, which extends the optimal brain compression framework~\citep{frantar2022optimal,hassibi1992second} to LLM scale by processing weights column by column with Hessian-based error compensation. AWQ~\citep{lin2024awq} identifies salient weight channels via activation magnitudes and applies per-channel scaling before RTN. \crnew{LDLQ~\citep{chee2023quip}, the adaptive-rounding procedure introduced in QuIP and subsequently used in QuIP\#~\citep{tseng2024quipsharp}, implements rounding with linear feedback via the LDL decomposition of the Hessian.} \ourmethod builds on top of the data-aware layer-wise quantization approach. However, it differs from prior methods by using a joint ADMM optimization procedure as opposed to greedy sequential quantization for GPTQ/LDLQ and data-free methods like RTN.

\noindent\textbf{Rotation and scaling transformations.} Empirically LLM activations have a high presence of outliers \cite{xiao2023smoothquant}, which makes them difficult to quantize online. Several methods reduce quantization difficulty through equivalent model transformations applied before the weight quantizer. SmoothQuant~\citep{xiao2023smoothquant} migrates quantization difficulty from activations to weights via per-channel scaling. QuaRot~\citep{ashkboos2024quarot} applies random Hadamard rotations to reduce the effect of outliers in the quantized network, while SpinQuant~\citep{liu2024spinquant} learns the rotations, followed by RTN or GPTQ weight quantization. These transformation methods are orthogonal to the choice of weight quantizer; in our experiments, we demonstrate that replacing RTN or GPTQ with \ourmethod gives consistent improvements (\cref{sec:experimental-results}).

\noindent\textbf{ADMM for LLM compression.} Prior work has identified that designing joint optimization approaches under model compression constraints can outperform more heuristic methods. Under sparsity constraints, ADMM~\citep{boyd2011distributed,davis2016convergence} has been exceptionally successful, outperforming prior heuristic methods such as SparseGPT (\cite{frantar2023sparsegpt}) and Wanda \cite{sun2024simpleeffectivepruningapproach}. For pure sparsity, \cite{bovza2024fast} employs ADMM to recover optimal weights on a fixed sparsity support, while \cite{meng2024alps} use ADMM for both support identification and weight optimization. Further, \cite{makni2025basil} use ADMM to sparse-plus-low-rank decomposition.
However, ADMM under quantization constraints remains largely unexplored for LLM quantization due to the scale of LLMs and the hard combinatorial structure of the optimization constraints. ~\citet{pmlr-v130-huang21a} share the high-level idea of applying ADMM to quantization, but target a complementary setting.
They target ADMM for general discrete optimization and, in their main experiments, binarize vision networks from scratch against a global training loss.
In contrast, \ourmethod is a post-training LLM quantization method designed for the layer-wise reconstruction loss.

\vspace{-0.8em}
\section{ADMM-Q}
\label{sec:problem-formulation}
\vspace{-0.8em}

\subsection{Problem formulation}\label{sec:formulation}
\vspace{-0.5em}

As in GPTQ \citep{frantar2022gptq} and AWQ \citep{lin2024awq}, we compress the model layer by layer, minimizing the reconstruction error between the outputs of the pre-trained weights and the quantized weights on calibration data. For a given layer with pre-trained weight matrix $\widehat{\mathbf{W}} \in \mathbb{R}^{n \times p}$ ($n$ and $p$ correspond to input and output channels, respectively) and input activations $\mathbf{X} \in \mathbb{R}^{N \times n}$ collected from $N$ calibration samples, we solve:
\begin{equation}\label{eq:layer-recon}
\min_{\mathbf{W}} \quad \frac{1}{2}\|\mathbf{X}\widehat{\mathbf{W}} - \mathbf{X}\mathbf{W}\|_F^2 + \frac{\lambda}{2}\|\widehat{\mathbf{W}} - \mathbf{W}\|_F^2 \quad \text{s.t.} \quad \mathbf{W} \in \mathcal{Q}
\end{equation}
where $\mathcal{Q}$ denotes the set of quantized weight matrices (e.g., 4-bit integers with per-channel scale and zero-point) and $\lambda > 0$ is a regularization parameter. Defining the Hessian $\mathbf{H} = \mathbf{X}^\top\mathbf{X} + \lambda\mathbf{I}$, this is equivalent to:
\begin{equation}\label{eq:hessian-form}
\min_{\mathbf{W}} \quad \frac{1}{2}\mathrm{Tr}\left((\mathbf{W} - \widehat{\mathbf{W}})^\top \mathbf{H} (\mathbf{W} - \widehat{\mathbf{W}})\right) \quad \text{s.t.} \quad \mathbf{W} \in \mathcal{Q}
\end{equation}
where $\mathrm{Tr}(\cdot)$ denotes the trace of a matrix. This formulation captures the key challenge: finding quantized weights $\mathbf{W} \in \mathcal{Q}$ that minimize the Hessian-weighted distance to the pre-trained weights. Both RTN and GPTQ can be viewed as approximate solvers for this problem; RTN ignores $\mathbf{H}$ entirely (treating it as the identity), while GPTQ applies a greedy column-wise solver based on the optimal brain compression framework~\citep{hassibi1992second,frantar2022optimal}. 
ADMM is natural because it separates the smooth Hessian-weighted objective (continuous weight updates) from the discrete quantization constraint (enforcing the constraints gradually). 

\vspace{-0.75em}
\subsection{Diagonal scaling for numerical stability}\label{sec:scaling}
\vspace{-0.5em}
In practice, the activation norms $\|\mathbf{X}_{:,i}\|_2$ vary significantly across input dimensions, leading to a poorly conditioned Hessian $\mathbf{H}$ and unstable optimization. This issue is particularly severe in LLMs, where a small fraction of channels carry disproportionately large activation magnitudes, which is the well-documented ``channel outlier'' phenomenon~\citep{xiao2023smoothquant, dettmers2023spqr}. GPTQ sidesteps this issue in part because its greedy  procedure conditions on a single diagonal entry of the inverse Hessian at a time, so the condition number is less directly harmful. By contrast, \ourmethod jointly updates all weights simultaneously through a closed-form solution involving the full Hessian, making it directly sensitive to the condition number of~$\mathbf{H}$. Without preconditioning, outlier channels with large Hessian entries dominate the update, causing slow convergence and numerical instability.  We introduce a diagonal scaling that addresses this. Let $\bm{\Sigma} = \operatorname{Diag}(\mathbf{H})^{-1/2}$ and define the scaled variable $\mathbf{W}' = \bm{\Sigma}^{-1}\mathbf{W}$. Problem~\eqref{eq:hessian-form} becomes:
\begin{equation}\label{eq:scaled-form}
\min_{\mathbf{W}'} \quad \frac{1}{2}\mathrm{Tr}\Big((\mathbf{W}' - \bm{\Sigma}^{-1}\widehat{\mathbf{W}})^\top \underbrace{(\bm{\Sigma}^\top \mathbf{H}\, \bm{\Sigma})}_{\text{scaled Hessian}} (\mathbf{W}' - \bm{\Sigma}^{-1}\widehat{\mathbf{W}})\Big) \quad \text{s.t.} \quad \bm{\Sigma}\mathbf{W}' \in \mathcal{Q}
\end{equation}
The scaled Hessian $\bm{\Sigma}^\top\mathbf{H}\bm{\Sigma}$ has unit diagonal entries, which substantially reduces its condition number. This ensures all channels contribute comparably to the ADMM updates. For notational simplicity, we use $\mathbf{W}$, $\widehat{\mathbf{W}}$, and $\mathbf{H}$ to denote the scaled quantities, writing:
\begin{equation}\label{eq:admm-main}
\min_{\mathbf{W}} \quad \frac{1}{2}\mathrm{Tr}\left((\mathbf{W} - \widehat{\mathbf{W}})^\top \mathbf{H} (\mathbf{W} - \widehat{\mathbf{W}})\right) \quad \text{s.t.} \quad \bm{\Sigma}\mathbf{W} \in \mathcal{Q}
\end{equation}

In our ablation study (Table~\ref{tab:admmq-ablation-qwen3-8b-w3w4}), we find that diagonal preconditioning is one of the most impactful components of \ourmethod: removing it at W3A16 on Qwen3-8B increases WikiText-2 perplexity from 10.06 to 14.68, worse than the 12.85 of GPTQ. 

\vspace{-0.3em}
\subsection{ADMM algorithm}\label{sec:admm-algorithm}
\vspace{-0.3em}

To solve the non-convex problem~\eqref{eq:admm-main}, we employ ADMM. We introduce an auxiliary variable $\mathbf{D}$ to decouple the continuous optimization from the quantization constraint:
\begin{equation}\label{eq:admm-split}
\min_{\mathbf{W}, \mathbf{D}} \quad \frac{1}{2}\mathrm{Tr}\left((\mathbf{W} - \widehat{\mathbf{W}})^\top \mathbf{H} (\mathbf{W} - \widehat{\mathbf{W}})\right) + \mathbb{I}_{\mathcal{Q}}(\bm{\Sigma}\mathbf{D}) \quad \text{s.t.} \quad \mathbf{W} = \mathbf{D}
\end{equation}
where $\mathbb{I}_{\mathcal{Q}}(\cdot)$ is the indicator function for the quantization set. The augmented Lagrangian with dual variable $\mathbf{V}$ and penalty $\rho > 0$ is:
\begin{equation}\label{eq:aug-lagrangian}
L_\rho(\mathbf{W}, \mathbf{D}, \mathbf{V}) = \frac{1}{2}\mathrm{Tr}\left((\mathbf{W} - \widehat{\mathbf{W}})^\top \mathbf{H} (\mathbf{W} - \widehat{\mathbf{W}})\right) + \mathbb{I}_{\mathcal{Q}}(\bm{\Sigma}\mathbf{D}) + \langle \mathbf{V}, \mathbf{W} - \mathbf{D} \rangle + \frac{\rho}{2}\|\mathbf{W} - \mathbf{D}\|_F^2
\end{equation}
The ADMM updates at iteration $t$ are:
\begin{align}
\mathbf{W}^{(t+1)} &= (\mathbf{H} + \rho_t \mathbf{I})^{-1}(\mathbf{H}\widehat{\mathbf{W}} + \rho_t \mathbf{D}^{(t)} - \mathbf{V}^{(t)}) \label{eq:w-update} \\
\mathbf{D}^{(t+1)} &= \argmin_{\bm{\Sigma}\mathbf{D} \in \mathcal{Q}} \|\mathbf{D} - (\mathbf{W}^{(t+1)} + \mathbf{V}^{(t)}/\rho_t)\|_F^2 \label{eq:d-update} \\
\mathbf{V}^{(t+1)} &= \mathbf{V}^{(t)} + \rho_t(\mathbf{W}^{(t+1)} - \mathbf{D}^{(t+1)}) \label{eq:v-update}
\end{align}

\noindent\textbf{Efficient $\mathbf{W}$-update.}~~By pre-computing the eigendecomposition $\mathbf{H} = \mathbf{U}\Lambda\mathbf{U}^\top$ once, the matrix inverse in~\eqref{eq:w-update} can be evaluated as $(\mathbf{H} + \rho_t\mathbf{I})^{-1} = \mathbf{U}(\Lambda + \rho_t\mathbf{I})^{-1}\mathbf{U}^\top$ for any $\rho_t$, reducing each iteration to matrix multiplications.

\noindent\textbf{$\mathbf{D}$-update (quantization projection).}~~The closed form solution for the $\mathbf{D}$-update~\eqref{eq:d-update} reduces to RTN; it projects onto the quantization grid, accounting for the scaling $\bm{\Sigma}$. For a fixed uniform grid, each block/channel/tensor is solved independently. Given a target $\mathbf{\widetilde{W}}^{(t)} = \mathbf{W}^{(t+1)} + \mathbf{V}^{(t)}/\rho_t$, the quantized values are:
\begin{equation}\label{eq:projection}
\bm{\Sigma}\mathbf{D} = \Delta \cdot \operatorname{Round}\!\left(\frac{\bm{\Sigma}\mathbf{\widetilde{W}}^{(t)} - z_0}{\Delta}\right) + z_0
\end{equation}
where $\Delta$ is the quantization scale and $z_0$ is the zero-point. 

\paragraph{Grid refresh.}
Let $\theta=(\Delta, z_0)$ denote the scale and zero-point determining the quantization grid $\mathcal{Q}(\theta)$, and notice that $\mathbf{\widetilde{W}}^{(t)}$ is the continuous point that is projected in the $\mathbf{D}$-update. As in GPTQ, we initialize $\theta^{(0)}$ by RTN on the dense weights, but as ADMM moves the continuous iterate, this fixed grid can become stale for the current target $\mathbf{\widetilde{W}}^{(t)}$.
\par
\begin{wrapfigure}[16]{r}{0.52\linewidth}
\centering
\begin{tikzpicture}[
    x=0.52cm,
    y=0.48cm,
    >=Latex,
    every node/.style={font=\tiny}
]

  \draw[->, thick, black!70] (-0.4,0) -- (10.2,0)
    node[right, font=\tiny] {$w_i$};
  \draw[->, thick, black!70] (0,-0.4) -- (0,8.6)
    node[above, font=\tiny] {$w_j$};

  \foreach \x in {1.0,3.0,5.0,7.0} {
    \draw[gray!22, very thin] (\x,0.3) -- (\x,7.8);
  }
  \foreach \y in {1.2,3.2,5.2,7.2} {
    \draw[gray!22, very thin] (0.3,\y) -- (8.0,\y);
  }
  \foreach \x in {1.0,3.0,5.0,7.0} {
    \foreach \y in {1.2,3.2,5.2,7.2} {
      \fill[gray!55] (\x,\y) circle (1.8pt);
    }
  }

  \foreach \x in {2.0,4.0,6.0,8.0} {
    \draw[blue!18, very thin] (\x,0.3) -- (\x,7.8);
  }
  \foreach \y in {0.8,2.8,4.8,6.8} {
    \draw[blue!18, very thin] (1.3,\y) -- (9.0,\y);
  }
  \foreach \x in {2.0,4.0,6.0,8.0} {
    \foreach \y in {0.8,2.8,4.8,6.8} {
      \fill[blue!55!cyan] (\x,\y) circle (1.8pt);
    }
  }

  \coordinate (Wold) at (3.35,3.55);
  \coordinate (Wnew) at (6.15,4.95);
  \coordinate (Pold) at (3.0,3.2);
  \coordinate (Pstale) at (7.0,5.2);
  \coordinate (Pnew) at (6.0,4.8);

  \fill[red!65!white!85!black] (Wold) circle (2.4pt);
  \node[red!65!white!85!black, anchor=south east]
    at ($(Wold)+(-0.05,0.10)$) {$\mathbf{\widetilde{W}}^{(t-k)}$};

  \fill[red!65!white!85!black] (Wnew) circle (2.4pt);
  \node[red!65!white!85!black, anchor=south west]
    at ($(Wnew)+(-0.35,0.22)$) {$\mathbf{\widetilde{W}}^{(t)}$};

  \draw[->, red!65!white!85!black, line width=0.9pt]
    (Wold) to[bend left=8] (Wnew);

  \draw[fill=gray!58, draw=gray!75!black, line width=0.6pt]
    ($(Pold)+(-0.11,-0.11)$) rectangle ($(Pold)+(0.11,0.11)$);
  \draw[fill=gray!58, draw=gray!75!black, line width=0.6pt]
    ($(Pstale)+(-0.11,-0.11)$) rectangle ($(Pstale)+(0.11,0.11)$);
  \draw[fill=blue!42!cyan, draw=blue!65!black, line width=0.6pt]
    ($(Pnew)+(-0.11,-0.11)$) rectangle ($(Pnew)+(0.11,0.11)$);

  \draw[dashed, gray!68, line width=0.7pt] (Wold) -- (Pold);
  \draw[dashed, gray!68, line width=0.7pt] (Wnew) -- (Pstale);
  \draw[dashed, blue!70!black, line width=0.7pt] (Wnew) -- (Pnew);

  \node[gray!70!black, anchor=north east]
    at ($(Pold)+(0.80,-0.08)$) {$\operatorname{Proj}_{\mathcal{Q}(\theta)}(\mathbf{\widetilde{W}}^{(t-k)})$};
  \node[gray!70!black, anchor=west]
    at ($(Pstale)+(0.15,0.02)$) {$\operatorname{Proj}_{\mathcal{Q}(\theta)}(\mathbf{\widetilde{W}}^{(t)})$};
  \node[blue!65!black, anchor=north east]
    at ($(Pnew)+(2.35,-0.12)$) {$\operatorname{Proj}_{\mathcal{Q}(\theta^\prime)}(\mathbf{\widetilde{W}}^{(t)})$};

  \fill[gray!55] (0.9,8.1) circle (1.8pt);
  \node[gray!75!black, anchor=west] at (1.15,8.1) {old grid};
  \fill[blue!55!cyan] (3.2,8.1) circle (1.8pt);
  \node[blue!65!black, anchor=west] at (3.45,8.1) {refreshed grid};
  \fill[red!65!white!85!black] (6.5,8.1) circle (2.1pt);
  \node[red!65!white!85!black, anchor=west] at (6.8,8.1) {iterate};

\end{tikzpicture}
\caption{Grid refresh updates the projection to match the current continuous iterate.}
\label{fig:grid-refresh}
\end{wrapfigure}
We therefore equip \ourmethod with the ability to refit the grid parameters from the current continuous weights and re-project, while adjusting the dual variable to keep the ADMM state consistent. We visualize this effect in \Cref{fig:grid-refresh}: the gray projection continues with the old grid after the iterate has moved, while the blue projection shows the refreshed grid better aligned with the current continuous point. This is a key advantage of the \ourmethod algorithm compared to GPTQ, which is required to keep the grid fixed during optimization to avoid compromising prior greedy decisions. We give further details on the grid refresh in Appendix~\ref{app:grid-refresh}.

\vspace{-.5em}
\subsection{Integration with popular quantization transformations}\label{sec:integration}
\vspace{-0.5em}
\ourmethod modifies only the layer-wise weight quantization step, so it can be inserted into PTQ pipelines. For instance, one can apply equivalent transformations that preserve the layer input-output map, but change the coordinates in which the local reconstruction problem is posed e.g. scaling or rotation; or other techniques such as weight clipping. 

\begin{itemize}[leftmargin=*,itemsep=1pt,topsep=2pt]
\item \textbf{Scaling} (SmoothQuant~\citep{xiao2023smoothquant}, AWQ~\citep{lin2024awq}): Per-channel scaling can be absorbed into \ourmethod's scaling matrix \(\bm{\Sigma}\), giving the same layer-wise problem with a modified scaling parameterization.

\item \textbf{Rotation} (SpinQuant~\citep{liu2024spinquant}, QuaRot~\citep{ashkboos2024quarot}): If rotation matrices are first applied, then \ourmethod is applied to the corresponding transformed layer-wise reconstruction problem, i.e., the same local quantization objective expressed in the rotated coordinates with the associated Hessian.

\item \textbf{Clipping Heuristics.} The scale and zero-point in \eqref{eq:projection} can be set without clipping as follows: $\Delta = (\max - \min)/(2^b - 1)$. Heuristics for learning quantization grid clipping can also be integrated within our framework. In this paper, we employ the default method used in GPTQ and SpinQuant: MSE grid search for clipping ratios using 100 grid points with maximum shrink ratio 0.8 (see Appendix~\ref{app:appendix-exp-details}).
\end{itemize}

Appendix~\ref{app:equiv-transform} gives the form of the modified reconstruction problem for scaling and rotation. 


\begin{algorithm}[t]
\caption{\ourmethod: ADMM-based Weight Quantization}
\label{alg:admm}
\begin{algorithmic}[1]
\Require Pre-trained weights $\widehat{\mathbf{W}}$, Hessian $\mathbf{H} = \mathbf{X}^\top\mathbf{X} + \lambda\mathbf{I}$, bit-width $b$, iterations $T$, initial penalty $\rho_0$, growth factor $\gamma > 1$
\Ensure Quantized weights $\mathbf{W}_q$
\State $\bm{\Sigma} \gets \operatorname{Diag}(\mathbf{H})^{-1/2}$ \Comment{Scaling matrix}
\State $\widehat{\mathbf{W}} \gets \bm{\Sigma}^{-1}\widehat{\mathbf{W}}$, \quad $\mathbf{H} \gets \bm{\Sigma}^\top \mathbf{H}\, \bm{\Sigma}$ \Comment{Scale problem}
\State $\mathbf{U}\Lambda\mathbf{U}^\top \gets \text{eig}(\mathbf{H})$ \Comment{Eigendecomposition}
\State $\mathbf{D}^{(0)} \gets \widehat{\mathbf{W}}$, \quad $\mathbf{V}^{(0)} \gets \mathbf{0}$
\For{$t = 0, 1, \ldots, T-1$}
    \State $\mathbf{W}^{(t+1)} \gets \mathbf{U}(\Lambda + \rho_t \mathbf{I})^{-1}\mathbf{U}^\top(\mathbf{H}\widehat{\mathbf{W}} + \rho_t \mathbf{D}^{(t)} - \mathbf{V}^{(t)})$ \Comment{W-update}
    \State $\mathbf{D}^{(t+1)} \gets \bm{\Sigma}^{-1}\operatorname{Proj}_\mathcal{Q}\!\left(\bm{\Sigma}(\mathbf{W}^{(t+1)} + \mathbf{V}^{(t)}/\rho_t)\right)$ \Comment{Quantize}
    \State $\mathbf{V}^{(t+1)} \gets \mathbf{V}^{(t)} + \rho_t(\mathbf{W}^{(t+1)} - \mathbf{D}^{(t+1)})$ \Comment{Dual update}
    \State $\rho_{t+1} \gets \texttt{update}(\rho_t)$
\EndFor
\State $\mathbf{W}_q \gets \bm{\Sigma}\mathbf{D}^{(T)}$
\State $\mathbf{W}_q \gets \Call{PairSwap}{\mathbf{W}_q, \widehat{\mathbf{W}}, \mathbf{H}}$ \Comment{Local search (\cref{sec:local-search})}
\State \Return $\mathbf{W}_q$
\end{algorithmic}
\end{algorithm}
\vspace{-0.75em}
\subsection{Local search refinement}\label{sec:local-search}
\vspace{-0.5em}

Since ADMM updates all weights jointly, the $\mathbf{W}$-update may shift many weights in the same direction. Weights near the midpoint of two adjacent grid points are particularly susceptible. This correlated rounding might lead to suboptimal discrete solutions. To reconcile this issue, we propose a highly-efficient \emph{pair-swap} local search as a post-processing step. Once \ourmethod has converged, the solution $\mathbf{W}_q$ may be improved by discrete local moves that directly reduce the reconstruction error~\eqref{eq:hessian-form}.
\vspace{-0.75em}
\paragraph{Setup.}
Recall $\mathbf{W}_q \in \mathbb{R}^{n \times p}$ (defined in \cref{alg:admm}).
Let $\Delta_j$ be the per-output-channel scale for column~$j$. Each entry $[\mathbf{W}_q]_{ij}$ takes values on the grid $\{z_j+k\,\Delta_j : k \in \mathbb{Z},\, 0\le k\le 2^b-1\}$. A \emph{pair-swap} selects two rows $i_1, i_2$ and, for each column~$j$ independently, shifts $[\mathbf{W}_q]_{i_1 j}$ by $\pm\Delta_j$ and $[\mathbf{W}_q]_{i_2 j}$ by $\pm\Delta_j$, subject to staying within the grid bounds.
\vspace{-0.5em}
\paragraph{Closed-form cost.}
The reconstruction error~\eqref{eq:hessian-form} decomposes as a sum over output columns: $\mathcal{L} = \sum_{j=1}^{p} \boldsymbol{\epsilon}_j^\top \mathbf{H}\, \boldsymbol{\epsilon}_j$, where $\boldsymbol{\epsilon}_j = (\mathbf{W}_q - \widehat{\mathbf{W}})_{:,j} \in \mathbb{R}^{n}$. Perturbing rows $i_1, i_2$ in column~$j$ by $\delta_{i_1 j}, \delta_{i_2 j} \in \{+\Delta_j, -\Delta_j\}$ changes the loss contribution of column~$j$ by
\begin{equation}\label{eq:delta-loss}
\Delta\mathcal{L}_{i_1, i_2, j}
= -2\,\delta_{i_1 j}\, G_{i_1 j}
+ \delta_{i_1 j}^2\, H_{i_1 i_1}
- 2\,\delta_{i_2 j}\, G_{i_2 j}
+ \delta_{i_2 j}^2\, H_{i_2 i_2}
+ 2\,\delta_{i_1 j}\delta_{i_2 j}\, H_{i_1 i_2}\,,
\end{equation}
where $\mathbf{G} = \mathbf{H}(\widehat{\mathbf{W}} - \mathbf{W}_q)$ is the gradient of the objective. Only three Hessian entries and two gradient entries are needed are needed to evaluate swapping a pair with column $j$ and rows $i_1$ and $i_2$. Evaluating all potential swaps is made extremely efficient through GPU parallelization in our implementation.
\vspace{-1em}
\paragraph{Algorithm.}
For a given row pair $(i_1, i_2)$, each output column~$j$  selects the best sign combination among $\{(\pm\Delta_j, \pm\Delta_j)\}$ via~\eqref{eq:delta-loss}. First, we compute $\mathbf{G}$ (one matrix-multiplication cost). Then each round proceeds as follows: (i)~sample random row pairs in batches; (ii)~evaluate~\eqref{eq:delta-loss} for all four sign combinations per column; (iv)~apply the pair with the largest total improvement; (v)~update $\mathbf{G}$ via a rank-2 correction: $\mathbf{G} \gets \mathbf{G} - \mathbf{H}_{:,\{i_1, i_2\}}\,[\boldsymbol{\delta}_{i_1,:};\, \boldsymbol{\delta}_{i_2,:}]$, where $\boldsymbol{\delta}_{i,:} \in \mathbb{R}^{1 \times p}$ is the applied perturbation for row~$i$. We run at most 5 rounds. Our ablation studies (Table~\ref{tab:admmq-ablation-spinquant-w3w4}) show that local search is especially impactful on LLaMA models within SpinQuant at W4A4KV4, where removing it increases WikiText-2 perplexity from 7.26 to 7.97 on LLaMA-3 8B and from 15.55 to 19.45 on LLaMA-3.2 1B-Instruct.

\vspace{-0.75em}
\subsection{Convergence}\label{sec:convergence}
\vspace{-0.5em}
\begin{thm}\label{thm:admmq}
Let $\{\mathbf{D}^{(t)}\}_{t=0}^\infty$ and $\{\mathbf{W}^{(t)}\}_{t=0}^\infty$ be the sequences generated by \cref{alg:admm}. Suppose the penalty parameters $\{\rho_t\}_{t=1}^\infty$ satisfy $\sum_{t=1}^\infty 1/\rho_t < \infty$. Then for any $t \ge 1$:
\begin{equation}
\max\left\{\|\mathbf{D}^{(t+1)} - \mathbf{D}^{(t)}\|_F,\; \|\mathbf{W}^{(t+1)} - \mathbf{D}^{(t+1)}\|_F\right\} \le \frac{C}{\rho_t}
\end{equation}
where $C$ is a constant depending on $\mathbf{X}$, $\widehat{\mathbf{W}}$, $\lambda$, and $\sum_{t=1}^\infty 1/\rho_t$. In particular, there exists a quantized matrix $\bar{\mathbf{D}}$ such that $\mathbf{D}^{(t)} \to \bar{\mathbf{D}}$ and $\mathbf{W}^{(t)} \to \bar{\mathbf{D}}$ as $t \to \infty$.
\end{thm}
\vspace{-0.7em}
The condition $\sum 1/\rho_t < \infty$ is satisfied by, e.g., $\rho_t = \rho_0 \cdot \gamma^t$ with $\gamma > 1$, which is the schedule used in \cref{alg:admm}. This guarantees that the continuous and discrete iterates converge to the same limit point, ensuring the constraint $\mathbf{W} = \mathbf{D}$ is asymptotically satisfied. Theorem \ref{thm:admmq} extends the convergence guarantee of ADMM in the sparsity setting \cite{meng2024alps} to quantization, including under grid refresh. The proof is deferred to Appendix \ref{app:proofadmm}.

\vspace{-0.7em}
\section{Experimental results}
\label{sec:experimental-results}
\vspace{-0.5em}
\subsection{Experimental setup} \label{sect:exp-setup}
\vspace{-0.5em}
\noindent\textbf{Models and evaluation protocol.}~~~
We evaluate \ourmethod across three settings: (i)~weight-only quantization on the Qwen3 model family (1.7B, 4B, 8B, 14B, 32B); (ii)~weight-and-activation quantization with SpinQuant~\citep{liu2024spinquant} on LLaMA-3 8B~\citep{dubey2024llama}, LLaMA-3.2 1B-Instruct, and Qwen3-8B; and (iii)~weight-and-activation quantization with SmoothQuant~\citep{xiao2023smoothquant} on Qwen3-8B.
We report perplexity ($\downarrow$) on WikiText-2, C4, and Penn Treebank (PTB), and zero-shot accuracy ($\uparrow$) averaged over nine tasks. For zero-shot evaluation, we use the LM Evaluation Harness~\citep{gao10256836framework} on PIQA~\citep{bisk2020piqa}, ARC-Easy/Challenge~\citep{clark2018think}, HellaSwag~\citep{zellers2019hellaswag}, Winogrande~\citep{sakaguchi2021winogrande}, RTE~\citep{wang2018glue}, OpenBookQA~\citep{mihaylov2018openbookqa}, BoolQ~\citep{clark2019boolq}, and Social IQa~\citep{sap2019socialiqacommonsensereasoningsocial}. 

\vspace{-0.8em}
\subsection{Numerical results}
\vspace{-0.8em}

\paragraph{Weight-only quantization on the Qwen3 model family.}
Table~\ref{tab:qwen3-architecture-sweep} compares GPTQ, AWQ, \crnew{LDLQ,} and \ourmethod under weight-only W4/3/2 per-channel quantization across five Qwen3 model sizes (1.7B--32B). \ourmethod consistently improves over GPTQ on WikiText-2, C4, and PTB perplexity across bit-widths, with larger gains at W3 and W2 where the quantization problem is more challenging. For example, on Qwen3-8B at W3, \ourmethod reduces WikiText-2 perplexity from 12.85 to 10.06 and PTB perplexity from 35.01 to 27.30. Zero-shot accuracy follows a similar trend. AWQ is strong at W4 but less competitive at W3 and W2, perhaps because it uses RTN as its weight quantizer. \crnew{Relative to LDLQ, \ourmethod matches or improves at W4 and substantially improves at W3 and W2.} To understand the source of improvements, Figure~\ref{fig:layer-recon-error} compares the layer-wise reconstruction error of \ourmethod relative to GPTQ on Qwen3-8B under W4 and W3 per-channel quantization. \ourmethod achieves lower reconstruction error across all 36 decoder layers, with largest relative gains in early layers, where \ourmethod reduces reconstruction error to as low as 10\% of GPTQ's. In later layers, \ourmethod consistently retains roughly 70--75\% of GPTQ's error, both at W3 and W4. In Appendix~\ref{app:appendix-exp-details}, Table~\ref{tab:qwen3_runtime_memory} we report wall-clock quantization time and peak GPU memory usage for GPTQ and \ourmethod across all Qwen3 models. The runtime grows with model size, but remains practical: even at 32B, quantization completes in under two hours on a single GPU. We also measure peak GPU memory during quantization; GPTQ is 6\% less memory-intensive. Since weight quantization is a one-time step, we view the modest increase in runtime and memory as acceptable tradeoffs for the consistent gains in perplexity/zero-shot accuracy.

\vspace{-3mm}
\paragraph{Weight-and-activation quantization with SpinQuant and SmoothQuant.}
We evaluate \ourmethod as a drop-in replacement for GPTQ within the SpinQuant~\citep{liu2024spinquant} pipeline on Qwen3-8B, LLaMA-3 8B, and LLaMA-3.2 1B-Instruct and the SmoothQuant \cite{xiao2023smoothquant} pipeline on Qwen3-8B. 
SpinQuant applies learned rotation matrices ($R_1$, $R_2$) to the model weights before quantization, and optionally uses online Hadamard transforms ($R_3$, $R_4$) for activation and KV cache quantization. We compare GPTQ and \ourmethod across three weight bit-widths (2, 3, 4) with 4-bit activation and KV cache quantization (W$b$A4KV4), using four weight quantizer settings: symmetric/asymmetric $\times$ per-channel/groupsize-128. For GPTQ and ADMMQ we use the default configuration for rotation optimization W16A4KV4 rotation optimization for (W$b$A4KV4) quantization. For RTN, we use a single rotation optimized at W4A4KV4. Results are reported in \cref{tab:sq-llama-results,tab:qwen3-results}.

\vspace{-2mm}
\noindent\scalebox{0.96}{\begin{minipage}{1.042\textwidth}
\centering
\scriptsize
\setlength{\tabcolsep}{3.0pt}
\renewcommand{\arraystretch}{0.88}
\setlength{\abovecaptionskip}{3pt}
\setlength{\belowcaptionskip}{0pt}
\resizebox{\textwidth}{!}{%
\begin{tabular}{l|c|cccc|c|cccc|c|cccc|c|cccc}
\toprule
& \multicolumn{5}{c|}{\textbf{Wiki-2 PPL} ($\downarrow$)}
& \multicolumn{5}{c|}{\textbf{C4 PPL} ($\downarrow$)}
& \multicolumn{5}{c|}{\textbf{PTB PPL} ($\downarrow$)}
& \multicolumn{5}{c}{\textbf{0-shot Acc.} ($\uparrow$)} \\
\cmidrule(lr){2-6} \cmidrule(lr){7-11} \cmidrule(lr){12-16} \cmidrule(lr){17-21}
\textbf{Model}
& \cellcolor{gray!10}\textbf{BF16} & \texttt{GPTQ} & \texttt{AWQ} & \crnew{\texttt{LDLQ}} & \cellcolor{darkgreen!10}\ourmethod
& \cellcolor{gray!10}\textbf{BF16} & \texttt{GPTQ} & \texttt{AWQ} & \crnew{\texttt{LDLQ}} & \cellcolor{darkgreen!10}\ourmethod
& \cellcolor{gray!10}\textbf{BF16} & \texttt{GPTQ} & \texttt{AWQ} & \crnew{\texttt{LDLQ}} & \cellcolor{darkgreen!10}\ourmethod
& \cellcolor{gray!10}\textbf{BF16} & \texttt{GPTQ} & \texttt{AWQ} & \crnew{\texttt{LDLQ}} & \cellcolor{darkgreen!10}\ourmethod \\
\midrule
\multicolumn{21}{c}{\cellcolor{blue!8}\textbf{W4}} \\
\midrule
1.7B-Base & \cellcolor{gray!10}9.40 & 12.40 & \textbf{11.08} & \crnew{12.54} & \cellcolor{darkgreen!10}11.24
          & \cellcolor{gray!10}15.42 & 20.88 & \textbf{17.53} & \crnew{21.03} & \cellcolor{darkgreen!10}18.94
          & \cellcolor{gray!10}25.37 & 37.58 & \textbf{30.10} & \crnew{38.18} & \cellcolor{darkgreen!10}32.51
          & \cellcolor{gray!10}62.54 & 56.58 & \textbf{60.19} & \crnew{56.64} & \cellcolor{darkgreen!10}58.34 \\
4B-Base   & \cellcolor{gray!10}7.90 & 8.58  & 9.20  & \crnew{8.58} & \cellcolor{darkgreen!10}\textbf{8.55}
          & \cellcolor{gray!10}13.31 & 14.60 & 14.87 & \crnew{14.59} & \cellcolor{darkgreen!10}\textbf{14.57}
          & \cellcolor{gray!10}19.68 & \textbf{22.08} & 22.67 & \crnew{22.09} & \cellcolor{darkgreen!10}22.12
          & \cellcolor{gray!10}66.72 & 64.96 & 65.03 & \crnew{64.86} & \cellcolor{darkgreen!10}\textbf{65.63} \\
8B-Base   & \cellcolor{gray!10}7.00 & 7.82  & 8.06  & \crnew{7.83} & \cellcolor{darkgreen!10}\textbf{7.56}
          & \cellcolor{gray!10}11.91 & 13.39 & 13.14 & \crnew{13.37} & \cellcolor{darkgreen!10}\textbf{13.01}
          & \cellcolor{gray!10}17.12 & 19.82 & 19.26 & \crnew{19.92} & \cellcolor{darkgreen!10}\textbf{19.04}
          & \cellcolor{gray!10}69.07 & 68.93 & 67.95 & \crnew{\textbf{69.35}} & \cellcolor{darkgreen!10}68.14 \\
14B-Base  & \cellcolor{gray!10}6.38 & 6.87  & 7.28  & \crnew{6.89} & \cellcolor{darkgreen!10}\textbf{6.79}
          & \cellcolor{gray!10}10.99 & 11.91 & 12.05 & \crnew{11.92} & \cellcolor{darkgreen!10}\textbf{11.83}
          & \cellcolor{gray!10}15.30 & 16.81 & 16.99 & \crnew{16.81} & \cellcolor{darkgreen!10}\textbf{16.54}
          & \cellcolor{gray!10}71.75 & 70.76 & 69.82 & \crnew{71.06} & \cellcolor{darkgreen!10}\textbf{71.07} \\
32B       & \cellcolor{gray!10}7.61 & 8.22  & 8.44  & \crnew{8.18} & \cellcolor{darkgreen!10}\textbf{8.06}
          & \cellcolor{gray!10}12.45 & 13.58 & 13.54 & \crnew{\textbf{13.51}} & \cellcolor{darkgreen!10}13.55
          & \cellcolor{gray!10}18.82 & 20.84 & 21.18 & \crnew{20.90} & \cellcolor{darkgreen!10}\textbf{20.79}
          & \cellcolor{gray!10}71.49 & \textbf{69.98} & 69.68 & \crnew{69.27} & \cellcolor{darkgreen!10}69.14 \\
\midrule[\heavyrulewidth]
\multicolumn{21}{c}{\cellcolor{blue!8}\textbf{W3}} \\
\midrule
1.7B-Base & \cellcolor{gray!10}9.40 & 24.48 & 34.55 & \crnew{24.10} & \cellcolor{darkgreen!10}\textbf{16.39}
          & \cellcolor{gray!10}15.42 & 50.55 & 49.96 & \crnew{48.68} & \cellcolor{darkgreen!10}\textbf{37.96}
          & \cellcolor{gray!10}25.37 & 117.14 & 159.65 & \crnew{127.48} & \cellcolor{darkgreen!10}\textbf{52.95}
          & \cellcolor{gray!10}62.54 & 43.65 & 46.06 & \crnew{43.42} & \cellcolor{darkgreen!10}\textbf{46.78} \\
4B-Base   & \cellcolor{gray!10}7.90 & 13.64 & 21.30 & \crnew{13.50} & \cellcolor{darkgreen!10}\textbf{12.85}
          & \cellcolor{gray!10}13.31 & 24.84 & 28.29 & \crnew{24.56} & \cellcolor{darkgreen!10}\textbf{24.49}
          & \cellcolor{gray!10}19.68 & 41.42 & 60.86 & \crnew{41.59} & \cellcolor{darkgreen!10}\textbf{38.49}
          & \cellcolor{gray!10}66.72 & 50.89 & 49.74 & \crnew{51.50} & \cellcolor{darkgreen!10}\textbf{52.40} \\
8B-Base   & \cellcolor{gray!10}7.00 & 12.85 & 17.28 & \crnew{12.77} & \cellcolor{darkgreen!10}\textbf{10.06}
          & \cellcolor{gray!10}11.91 & 21.27 & 22.93 & \crnew{21.30} & \cellcolor{darkgreen!10}\textbf{18.25}
          & \cellcolor{gray!10}17.12 & 35.01 & 48.55 & \crnew{34.87} & \cellcolor{darkgreen!10}\textbf{27.30}
          & \cellcolor{gray!10}69.07 & 54.43 & 51.08 & \crnew{54.35} & \cellcolor{darkgreen!10}\textbf{57.80} \\
14B-Base  & \cellcolor{gray!10}6.38 & 9.50  & 14.34 & \crnew{9.47} & \cellcolor{darkgreen!10}\textbf{8.16}
          & \cellcolor{gray!10}10.99 & 18.03 & 19.50 & \crnew{17.84} & \cellcolor{darkgreen!10}\textbf{15.35}
          & \cellcolor{gray!10}15.30 & 30.12 & 39.31 & \crnew{30.05} & \cellcolor{darkgreen!10}\textbf{22.14}
          & \cellcolor{gray!10}71.75 & 64.73 & 53.45 & \crnew{63.94} & \cellcolor{darkgreen!10}\textbf{66.53} \\
32B       & \cellcolor{gray!10}7.61 & 10.70 & 16.82 & \crnew{10.85} & \cellcolor{darkgreen!10}\textbf{9.82}
          & \cellcolor{gray!10}12.45 & 18.65 & 23.06 & \crnew{19.03} & \cellcolor{darkgreen!10}\textbf{17.76}
          & \cellcolor{gray!10}18.82 & 32.13 & 65.32 & \crnew{33.75} & \cellcolor{darkgreen!10}\textbf{28.70}
          & \cellcolor{gray!10}71.49 & 62.92 & 60.28 & \crnew{63.93} & \cellcolor{darkgreen!10}\textbf{65.75} \\
\midrule[\heavyrulewidth]
\multicolumn{21}{c}{\cellcolor{blue!8}\textbf{W2}} \\
\midrule
1.7B-Base & \cellcolor{gray!10}9.40 & 441.72 & $2.5\text{e}7$ & \crnew{442.02} & \cellcolor{darkgreen!10}\textbf{42.72}
          & \cellcolor{gray!10}15.42 & 1301.16 & $2.7\text{e}7$ & \crnew{1.28e3} & \cellcolor{darkgreen!10}\textbf{252.47}
          & \cellcolor{gray!10}25.37 & 1472.54 & $2.2\text{e}7$ & \crnew{1.63e3} & \cellcolor{darkgreen!10}\textbf{199.86}
          & \cellcolor{gray!10}62.54 & 36.74 & 37.37 & \crnew{36.52} & \cellcolor{darkgreen!10}\textbf{37.62} \\
4B-Base   & \cellcolor{gray!10}7.90 & 27.75 & $1.9\text{e}7$ & \crnew{29.78} & \cellcolor{darkgreen!10}\textbf{24.04}
          & \cellcolor{gray!10}13.31 & 72.68 & $2.1\text{e}7$ & \crnew{72.96} & \cellcolor{darkgreen!10}\textbf{66.07}
          & \cellcolor{gray!10}19.68 & 101.48 & $2.2\text{e}7$ & \crnew{110.28} & \cellcolor{darkgreen!10}\textbf{91.55}
          & \cellcolor{gray!10}66.72 & 42.54 & 36.61 & \crnew{42.64} & \cellcolor{darkgreen!10}\textbf{43.01} \\
8B-Base   & \cellcolor{gray!10}7.00 & 32.88 & $3.4\text{e}7$ & \crnew{35.00} & \cellcolor{darkgreen!10}\textbf{17.69}
          & \cellcolor{gray!10}11.91 & 89.37 & $2.8\text{e}7$ & \crnew{89.64} & \cellcolor{darkgreen!10}\textbf{45.31}
          & \cellcolor{gray!10}17.12 & 145.36 & $3.1\text{e}7$ & \crnew{149.94} & \cellcolor{darkgreen!10}\textbf{61.68}
          & \cellcolor{gray!10}69.07 & 41.40 & 37.86 & \crnew{41.21} & \cellcolor{darkgreen!10}\textbf{47.06} \\
14B-Base  & \cellcolor{gray!10}6.38 & 40.98 & $2.1\text{e}7$ & \crnew{42.31} & \cellcolor{darkgreen!10}\textbf{12.62}
          & \cellcolor{gray!10}10.99 & 125.90 & $1.9\text{e}7$ & \crnew{151.29} & \cellcolor{darkgreen!10}\textbf{36.45}
          & \cellcolor{gray!10}15.30 & 177.38 & $1.9\text{e}7$ & \crnew{182.13} & \cellcolor{darkgreen!10}\textbf{45.22}
          & \cellcolor{gray!10}71.75 & 41.67 & 37.50 & \crnew{41.61} & \cellcolor{darkgreen!10}\textbf{51.58} \\
32B       & \cellcolor{gray!10}7.61 & 35.95 & $5.1\text{e}6$ & \crnew{23.21} & \cellcolor{darkgreen!10}\textbf{18.15}
          & \cellcolor{gray!10}12.45 & 553.18 & $4.4\text{e}6$ & \crnew{4.69e3} & \cellcolor{darkgreen!10}\textbf{50.00}
          & \cellcolor{gray!10}18.82 & 147.86 & $3.9\text{e}6$ & \crnew{113.30} & \cellcolor{darkgreen!10}\textbf{66.01}
          & \cellcolor{gray!10}71.49 & 52.77 & 36.57 & \crnew{51.20} & \cellcolor{darkgreen!10}\textbf{55.65} \\
\bottomrule
\end{tabular}
}
\vspace{-1pt}
\captionof{table}{Performance of Qwen3 under 4-bit (W4), 3-bit (W3), and 2-bit (W2) symmetric per-channel weight-only quantization, comparing GPTQ, AWQ, \crnew{LDLQ,} and \ourmethod. Note: the 32B-Base model is not available so we use \url{https://huggingface.co/Qwen/Qwen3-32B}.}
\label{tab:qwen3-architecture-sweep}
\vspace{3em}
\includegraphics[width=0.9\linewidth]{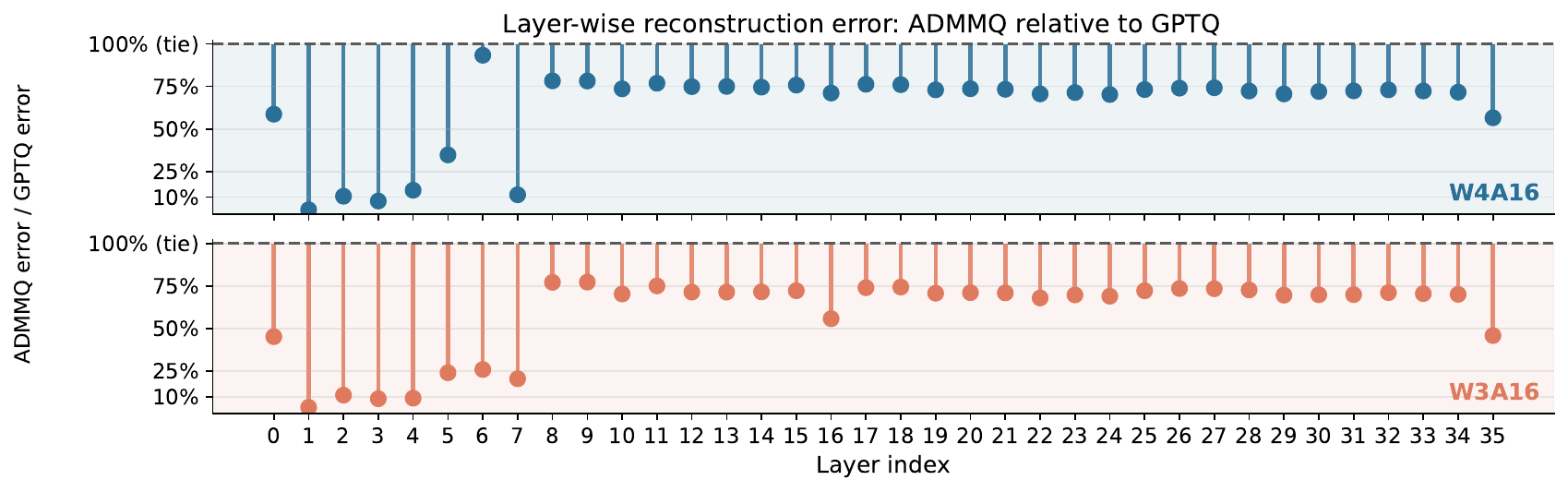}
\vspace{-0.6em}
\captionof{figure}{Layer-wise reconstruction error of \ourmethod relative to GPTQ on Qwen3-8B-Base under W4 (top) and W3 (bottom) per-channel weight-only quantization. Each point shows $\text{\ourmethod error} / \text{GPTQ error}$; values below the 100\% line indicate \ourmethod achieves lower error.}
\label{fig:layer-recon-error}
\end{minipage}}
\vspace{-4em}
\begin{table}[H]
\centering
\scriptsize
\setlength{\abovecaptionskip}{3pt}
\scalebox{0.88}{\begin{minipage}{1.136\columnwidth}
\centering
\resizebox{\columnwidth}{!}{%
\renewcommand{\arraystretch}{0.95}
\newcolumntype{C}[1]{>{\centering\arraybackslash}p{#1}}%
\begin{tabular}
{ll|C{2.8em}C{2.8em}C{2.8em}|C{2.8em}C{2.8em}C{2.8em}|C{2.8em}C{2.8em}C{2.8em}|C{2.8em}C{2.8em}C{2.8em}}
\toprule
& & \multicolumn{3}{c|}{\textbf{Wiki-2 PPL} ($\downarrow$)} & \multicolumn{3}{c|}{\textbf{C4 PPL} ($\downarrow$)} & \multicolumn{3}{c|}{\textbf{PTB PPL} ($\downarrow$)} & \multicolumn{3}{c}{\textbf{0-shot Acc.} ($\uparrow$)} \\
\cmidrule(lr){3-5} \cmidrule(lr){6-8} \cmidrule(lr){9-11} \cmidrule(lr){12-14}
\textbf{Config} & \textbf{Wt.\ Quant} & \texttt{RTN} & \texttt{GPTQ} & \cellcolor{darkgreen!10}\ourmethod & \texttt{RTN} & \texttt{GPTQ} & \cellcolor{darkgreen!10}\ourmethod & \texttt{RTN} & \texttt{GPTQ} & \cellcolor{darkgreen!10}\ourmethod & \texttt{RTN} & \texttt{GPTQ} & \cellcolor{darkgreen!10}\ourmethod \\
%
%
\midrule
\multicolumn{14}{c}{\cellcolor{blue!8}\textbf{LLaMA-3 8B}} \\
\midrule
\rowcolor{gray!10} \multicolumn{2}{l|}{\textbf{BF16 (dense)}} & \multicolumn{3}{c|}{6.14} & \multicolumn{3}{c|}{9.44} & \multicolumn{3}{c|}{14.84} & \multicolumn{3}{c}{63.94} \\
\midrule
\multirow{4}{*}{\textbf{W4A4KV4}}
  & S, per-ch  & 7.95 & 7.34 & \cellcolor{darkgreen!10}\textbf{7.26} & 13.44 & 12.16 & \cellcolor{darkgreen!10}\textbf{12.08} & 19.94 & 17.94 & \cellcolor{darkgreen!10}\textbf{17.74} & 58.55 & \textbf{60.55} & \cellcolor{darkgreen!10}60.52 \\
  & S, g128    & 8.99 & 7.27 & \cellcolor{darkgreen!10}\textbf{7.21} & 14.92 & 12.07 & \cellcolor{darkgreen!10}\textbf{11.98} & 22.00 & 17.75 & \cellcolor{darkgreen!10}\textbf{17.52} & 56.61 & 59.02 & \cellcolor{darkgreen!10}\textbf{60.17} \\
  & A, per-ch  & 8.11 & 7.35 & \cellcolor{darkgreen!10}\textbf{7.24} & 13.58 & 12.14 & \cellcolor{darkgreen!10}\textbf{12.05} & 19.97 & 17.92 & \cellcolor{darkgreen!10}\textbf{17.66} & 58.71 & 59.44 & \cellcolor{darkgreen!10}\textbf{60.26} \\
  & A, g128    & 9.22 & 7.26 & \cellcolor{darkgreen!10}\textbf{7.17} & 15.27 & 11.94 & \cellcolor{darkgreen!10}\textbf{11.86} & 21.59 & 17.74 & \cellcolor{darkgreen!10}\textbf{17.45} & 56.22 & 59.72 & \cellcolor{darkgreen!10}\textbf{60.86} \\
\midrule
\multirow{4}{*}{\textbf{W3A4KV4}}
  & S, per-ch  & 109.11 & 8.50 & \cellcolor{darkgreen!10}\textbf{8.02} & 117.26 & 15.01 & \cellcolor{darkgreen!10}\textbf{14.29} & 303.58 & 21.27 & \cellcolor{darkgreen!10}\textbf{20.13} & 38.01 & 56.51 & \cellcolor{darkgreen!10}\textbf{57.59} \\
  & S, g128    & 97.34 & 8.40 & \cellcolor{darkgreen!10}\textbf{8.01} & 112.94 & 14.79 & \cellcolor{darkgreen!10}\textbf{14.11} & 283.83 & 20.82 & \cellcolor{darkgreen!10}\textbf{19.86} & 38.81 & 56.08 & \cellcolor{darkgreen!10}\textbf{57.29} \\
  & A, per-ch  & 111.04 & 8.25 & \cellcolor{darkgreen!10}\textbf{8.05} & 137.82 & 14.34 & \cellcolor{darkgreen!10}\textbf{14.27} & 370.68 & 20.73 & \cellcolor{darkgreen!10}\textbf{20.24} & 37.69 & 56.31 & \cellcolor{darkgreen!10}\textbf{57.91} \\
  & A, g128    & 70.16 & 8.14 & \cellcolor{darkgreen!10}\textbf{7.91} & 66.87 & 13.93 & \cellcolor{darkgreen!10}\textbf{13.82} & 167.00 & 20.17 & \cellcolor{darkgreen!10}\textbf{19.68} & 40.85 & 57.14 & \cellcolor{darkgreen!10}\textbf{58.34} \\
\midrule
\multirow{4}{*}{\textbf{W2A4KV4}}
  & S, per-ch  & $1.6\text{e}6$ & 50.27 & \cellcolor{darkgreen!10}\textbf{15.65} & $1.7\text{e}6$ & 179.57 & \cellcolor{darkgreen!10}\textbf{45.33} & $1.9\text{e}6$ & 160.66 & \cellcolor{darkgreen!10}\textbf{74.80} & 35.47 & 39.52 & \cellcolor{darkgreen!10}\textbf{45.63} \\
  & S, g128    & $2.0\text{e}6$ & 29.36 & \cellcolor{darkgreen!10}\textbf{15.22} & $1.9\text{e}6$ & 89.24 & \cellcolor{darkgreen!10}\textbf{43.63} & $1.9\text{e}6$ & 114.14 & \cellcolor{darkgreen!10}\textbf{58.67} & 35.48 & 40.47 & \cellcolor{darkgreen!10}\textbf{44.64} \\
  & A, per-ch  & 5777 & 16.39 & \cellcolor{darkgreen!10}\textbf{14.95} & 4214 & \textbf{40.18} & \cellcolor{darkgreen!10}42.12 & 4289 & 60.55 & \cellcolor{darkgreen!10}\textbf{58.14} & 34.73 & 42.70 & \cellcolor{darkgreen!10}\textbf{44.79} \\
  & A, g128    & 6777 & 15.32 & \cellcolor{darkgreen!10}\textbf{14.02} & 6030 & 36.16 & \cellcolor{darkgreen!10}\textbf{35.68} & 5608 & 51.00 & \cellcolor{darkgreen!10}\textbf{47.44} & 34.53 & 42.44 & \cellcolor{darkgreen!10}\textbf{47.55} \\
\midrule[\heavyrulewidth]
\multicolumn{14}{c}{\cellcolor{blue!8}\textbf{LLaMA-3.2 1B-Instruct}} \\
\midrule
\rowcolor{gray!10} \multicolumn{2}{l|}{\textbf{BF16 (dense)}} & \multicolumn{3}{c|}{13.16} & \multicolumn{3}{c|}{21.30} & \multicolumn{3}{c|}{37.31} & \multicolumn{3}{c}{53.22} \\
\midrule
\multirow{4}{*}{\textbf{W4A4KV4}}
  & S, per-ch  & 15.97 & 15.81 & \cellcolor{darkgreen!10}\textbf{15.55} & 27.79 & 26.11 & \cellcolor{darkgreen!10}\textbf{25.80} & 53.59 & 48.69 & \cellcolor{darkgreen!10}\textbf{47.76} & 47.74 & 48.27 & \cellcolor{darkgreen!10}\textbf{49.31} \\
  & S, g128    & 19.33 & 15.67 & \cellcolor{darkgreen!10}\textbf{15.42} & 31.94 & 25.76 & \cellcolor{darkgreen!10}\textbf{25.34} & 62.57 & 47.55 & \cellcolor{darkgreen!10}\textbf{47.53} & 46.75 & 49.01 & \cellcolor{darkgreen!10}\textbf{49.72} \\
  & A, per-ch  & 17.09 & 15.74 & \cellcolor{darkgreen!10}\textbf{15.60} & 29.37 & 25.85 & \cellcolor{darkgreen!10}\textbf{25.69} & 55.43 & 48.59 & \cellcolor{darkgreen!10}\textbf{47.18} & 46.90 & 48.77 & \cellcolor{darkgreen!10}\textbf{49.38} \\
  & A, g128    & 20.05 & 15.50 & \cellcolor{darkgreen!10}\textbf{15.26} & 31.94 & 25.39 & \cellcolor{darkgreen!10}\textbf{25.36} & 60.66 & 47.39 & \cellcolor{darkgreen!10}\textbf{47.35} & 47.18 & 48.52 & \cellcolor{darkgreen!10}\textbf{49.36} \\
\midrule
\multirow{4}{*}{\textbf{W3A4KV4}}
  & S, per-ch  & 96.35 & 19.70 & \cellcolor{darkgreen!10}\textbf{17.88} & 159.31 & 34.07 & \cellcolor{darkgreen!10}\textbf{31.16} & 318.13 & 65.22 & \cellcolor{darkgreen!10}\textbf{56.65} & 38.61 & 45.76 & \cellcolor{darkgreen!10}\textbf{46.84} \\
  & S, g128    & 149.68 & 19.33 & \cellcolor{darkgreen!10}\textbf{18.06} & 227.31 & 33.50 & \cellcolor{darkgreen!10}\textbf{31.26} & 438.08 & 61.71 & \cellcolor{darkgreen!10}\textbf{55.99} & 37.54 & 45.62 & \cellcolor{darkgreen!10}\textbf{46.70} \\
  & A, per-ch  & 206.10 & 18.99 & \cellcolor{darkgreen!10}\textbf{18.05} & 355.45 & 32.43 & \cellcolor{darkgreen!10}\textbf{31.37} & 568.98 & 60.91 & \cellcolor{darkgreen!10}\textbf{57.47} & 36.99 & 45.78 & \cellcolor{darkgreen!10}\textbf{46.15} \\
  & A, g128    & 169.17 & 18.40 & \cellcolor{darkgreen!10}\textbf{17.55} & 209.44 & 30.90 & \cellcolor{darkgreen!10}\textbf{30.32} & 374.47 & 58.43 & \cellcolor{darkgreen!10}\textbf{56.59} & 37.14 & 45.88 & \cellcolor{darkgreen!10}\textbf{47.26} \\
\midrule
\multirow{4}{*}{\textbf{W2A4KV4}}
  & S, per-ch  & $2.5\text{e}6$ & 172.16 & \cellcolor{darkgreen!10}\textbf{43.98} & $2.1\text{e}6$ & 382.96 & \cellcolor{darkgreen!10}\textbf{123.33} & $2.5\text{e}6$ & 489.33 & \cellcolor{darkgreen!10}\textbf{157.51} & 36.67 & 36.85 & \cellcolor{darkgreen!10}\textbf{39.86} \\
  & S, g128    & $2.8\text{e}6$ & 131.90 & \cellcolor{darkgreen!10}\textbf{44.45} & $2.7\text{e}6$ & 370.65 & \cellcolor{darkgreen!10}\textbf{138.69} & $2.6\text{e}6$ & 440.34 & \cellcolor{darkgreen!10}\textbf{184.64} & 36.94 & 37.74 & \cellcolor{darkgreen!10}\textbf{38.67} \\
  & A, per-ch  & $1.2\text{e}4$ & 46.40 & \cellcolor{darkgreen!10}\textbf{41.84} & $1.2\text{e}4$ & 135.55 & \cellcolor{darkgreen!10}\textbf{123.04} & $1.1\text{e}4$ & 195.94 & \cellcolor{darkgreen!10}\textbf{166.27} & 34.58 & 38.20 & \cellcolor{darkgreen!10}\textbf{39.26} \\
  & A, g128    & $1.3\text{e}4$ & 46.76 & \cellcolor{darkgreen!10}\textbf{37.95} & $1.2\text{e}4$ & 118.29 & \cellcolor{darkgreen!10}\textbf{95.90} & $1.6\text{e}4$ & 174.98 & \cellcolor{darkgreen!10}\textbf{140.65} & 36.05 & \textbf{40.05} & \cellcolor{darkgreen!10}39.09 \\
\bottomrule
\end{tabular}%
}
\caption{RTN, GPTQ, and \ourmethod within the SpinQuant pipeline across two LLaMA models and weight bit-widths 2, 3, 4 (all W$b$A4KV4). ``S'' = symmetric, ``A'' = asymmetric; ``per-ch'' = per-channel, ``g128'' = groupsize 128. Best result per cell is \textbf{bolded}.}
\label{tab:sq-llama-results}
\end{minipage}}
\end{table}

\vspace{-5mm}

\begin{table}[H]
\centering
\scalebox{0.88}{%
\renewcommand{\arraystretch}{0.92}
\newcolumntype{C}[1]{>{\centering\arraybackslash}p{#1}}%
\setlength{\tabcolsep}{4pt}%
\small
\begin{tabular}{ll|C{3.2em}C{3.2em}|C{3.2em}C{3.2em}|C{3.2em}C{3.2em}|C{3.2em}C{3.2em}}
\toprule
& & \multicolumn{2}{c|}{\textbf{Wiki-2 PPL} ($\downarrow$)} & \multicolumn{2}{c|}{\textbf{C4 PPL} ($\downarrow$)} &
\multicolumn{2}{c|}{\textbf{PTB PPL} ($\downarrow$)} & \multicolumn{2}{c}{\textbf{0-shot Acc.} ($\uparrow$)} \\
\cmidrule(lr){3-4} \cmidrule(lr){5-6} \cmidrule(lr){7-8} \cmidrule(lr){9-10}
\textbf{Config} & \textbf{Wt.\ Quant} & \texttt{GPTQ} & \cellcolor{darkgreen!10}\ourmethod & \texttt{GPTQ} &
\cellcolor{darkgreen!10}\ourmethod & \texttt{GPTQ} & \cellcolor{darkgreen!10}\ourmethod & \texttt{GPTQ} &
\cellcolor{darkgreen!10}\ourmethod \\
\midrule[\heavyrulewidth]
\multicolumn{10}{c}{\cellcolor{blue!8}\textbf{Qwen3-8B (SmoothQuant, W$b$A8)}} \\[-0.6ex]
\midrule
\rowcolor{gray!10} \multicolumn{2}{l|}{\textbf{BF16 (dense)}} & \multicolumn{2}{c|}{7.00} & \multicolumn{2}{c|}{11.91} &
\multicolumn{2}{c|}{17.12} & \multicolumn{2}{c}{69.08} \\
\textbf{W4A8}
  & per-ch  & 9.29 & \cellcolor{darkgreen!10}\textbf{8.68} & 16.82 & \cellcolor{darkgreen!10}\textbf{15.69} & 25.60 &
\cellcolor{darkgreen!10}\textbf{23.44} & 65.74 & \cellcolor{darkgreen!10}\textbf{66.10} \\
\textbf{W3A8}
  & per-ch  & 16.00 & \cellcolor{darkgreen!10}\textbf{12.21} & 31.13 & \cellcolor{darkgreen!10}\textbf{24.98} & 50.62 &
\cellcolor{darkgreen!10}\textbf{38.14} & \textbf{54.91} & \cellcolor{darkgreen!10}54.64 \\
\textbf{W2A8}
  & per-ch  & 130.04 & \cellcolor{darkgreen!10}\textbf{18.89} & 204.11 & \cellcolor{darkgreen!10}\textbf{68.70} & 268.67 &
\cellcolor{darkgreen!10}\textbf{79.65} & 41.38 & \cellcolor{darkgreen!10}\textbf{46.13} \\
\midrule[\heavyrulewidth]
\multicolumn{10}{c}{\cellcolor{blue!8}\textbf{Qwen3-8B (SpinQuant, W$b$A4KV4)}} \\[-0.6ex]
\midrule
\rowcolor{gray!10} \multicolumn{2}{l|}{\textbf{BF16 (dense)}} & \multicolumn{2}{c|}{7.00} & \multicolumn{2}{c|}{11.91} &
\multicolumn{2}{c|}{17.12} & \multicolumn{2}{c}{69.08} \\
\textbf{W4A4KV4}
  & per-ch  & 11.80 & \cellcolor{darkgreen!10}\textbf{11.77} & 20.84 & \cellcolor{darkgreen!10}\textbf{20.82} & 37.64 &
\cellcolor{darkgreen!10}\textbf{37.58} & 57.67 & \cellcolor{darkgreen!10}\textbf{57.96} \\
\textbf{W3A4KV4}
  & per-ch  & 13.70 & \cellcolor{darkgreen!10}\textbf{13.11} & 25.57 & \cellcolor{darkgreen!10}\textbf{24.40} & 43.19 &
\cellcolor{darkgreen!10}\textbf{41.35} & 53.97 & \cellcolor{darkgreen!10}\textbf{54.70} \\
\textbf{W2A4KV4}
  & per-ch  & 66.11 & \cellcolor{darkgreen!10}\textbf{19.42} & 3324.38 & \cellcolor{darkgreen!10}\textbf{59.27} & 200.41 &
\cellcolor{darkgreen!10}\textbf{81.67} & 43.28 & \cellcolor{darkgreen!10}\textbf{46.29} \\
\bottomrule
\end{tabular}%
}
\caption{GPTQ vs.\ \ourmethod on Qwen3-8B-Base across:  weight-and-activation (W$b$A8), and SpinQuant (W$b$A4KV4). Best result per cell is \textbf{bolded}.}
\label{tab:qwen3-results}
\end{table}
\paragraph{Inference speed-ups in VLLM.} We integrate ADMM-Q into the llm-compressor/vLLM stack and benchmark inference latency at W8A8 under both SmoothQuant and SpinQuant on Qwen3 models (8B, 14B, 32B). Table \ref{tab:qwen_latency_speedup_gptq_admmq} reports speedups relative to the dense BF16 baseline, measured over 2000 Wikitext-2 sequences processed in batches of 100. We use 800 tokens of pre-fill and 256 tokens of decoding. Both GPTQ and ADMM-Q achieve nearly identical throughput gains: for example, on Qwen3-32B with SpinQuant W8A8, ADMM-Q achieves a 1.52× speedup, matching GPTQ's 1.52×. This is expected since both methods produce weight matrices in the same quantized format and differ only in the offline optimization procedure used to obtain them. We note that vLLM does not currently support W4A4 inference kernels on standard NVIDIA hardware, though hardware and kernel support for sub-8-bit weight-and-activation quantization is under active development. In Table \ref{tab:qwen3_runtime_memory}, we also provide inference memory and storage memory for GPTQ and \ourmethod, and show that they are also close to identical.

\begin{table}[H]
\centering
\small
\setlength{\tabcolsep}{5pt}
\begin{tabular}{lccc}
\toprule
\textbf{Method} & \textbf{8B-Base} & \textbf{14B-Base} & \textbf{32B} \\
\midrule
\rowcolor{gray!10}
\textbf{Dense BF16} 
& $1.0000$
& $1.0000$
& $1.0000$ \\
\midrule
\multicolumn{4}{c}{\cellcolor{blue!8}\textbf{SmoothQuant W8A8}} \\
\midrule
\texttt{GPTQ}
& $1.2393 \pm 0.0012$
& $1.3722 \pm 0.0010$
& \textbf{$1.4844 \pm 0.0023$} \\
\cellcolor{darkgreen!10}\ourmethod
& \cellcolor{darkgreen!10}\textbf{$1.2408 \pm 0.0012$}
& \cellcolor{darkgreen!10}\textbf{$1.3827 \pm 0.0029$}
& \cellcolor{darkgreen!10}$1.4808 \pm 0.0012$ \\
\midrule[\heavyrulewidth]
\multicolumn{4}{c}{\cellcolor{blue!8}\textbf{SpinQuant W8A8}} \\
\midrule
\texttt{GPTQ}
& \textbf{$1.2554 \pm 0.0014$}
& $1.3918 \pm 0.0008$
& $1.5158 \pm 0.0035$ \\
\cellcolor{darkgreen!10}\ourmethod
& \cellcolor{darkgreen!10}$1.2507 \pm 0.0015$
& \cellcolor{darkgreen!10}\textbf{$1.3977 \pm 0.0041$}
& \cellcolor{darkgreen!10}\textbf{$1.5220 \pm 0.0015$} \\
\bottomrule
\end{tabular}
\vspace{3pt}
\caption{Inference latency speedup relative to the dense BF16 baseline for Qwen3 models in VLLM under GPTQ and ADMMQ. Results are computed over 2000 sequences (processed in batches of 100). Values are reported as mean $\pm$ standard error.}
\label{tab:qwen_latency_speedup_gptq_admmq}
\end{table}

\subsection{Other quantization formats}
\label{sec:other-formats}
\crnew{
Because \ourmethod enforces discrete constraints only through its projection step, it can be combined with alternative quantization formats by changing the projection while leaving the continuous and dual updates unchanged.
ICQuant~\citep{li2025icquant} identifies per-channel outliers and quantizes outliers and inliers on separate grids. We fix the outlier set returned by ICQuant and apply the corresponding dual-grid projection in the $\mathbf D$-update.
Tables~\ref{tab:icquant-1p7b} and~\ref{tab:icquant-4b} in Appendix~\ref{app:nonuniform-reasoning} show that \ourmethod+ICQuant improves over both GPTQ+ICQuant and standalone \ourmethod on Qwen3-1.7B and 4B, especially at W2/W3.
SqueezeLLM~\citep{kim2024squeezellmdenseandsparsequantization} combines a sensitivity-based non-uniform codebook with a dense-and-sparse decomposition.
Using the same codebook and sparse mask, \ourmethod replaces nearest-neighbor rounding on the dense entries by ADMM updates.
Tables~\ref{tab:squeezellm-w3} and~\ref{tab:squeezellm-w2} report consistent gains over SqueezeLLM alone across Qwen3 (1.7B--32B), with the largest improvements at W2 (e.g., Qwen3-8B WikiText-2 / C4 from $15.07/24.94$ to $9.66/19.16$).
}

\subsection{Reasoning evaluations}
\label{sec:reasoning}
\crnew{
Perplexity and standard zero-shot commonsense benchmarks do not fully capture quantized model behavior on harder reasoning tasks. We therefore evaluate Qwen3-8B (the post-trained reasoning checkpoint) under asymmetric weight-only quantization with group size 128 on AIME-120, MATH-500, and GSM8K, following the protocol of~\citet{liu2025quantizationhurtsreasoningempirical}. Full results are in Tables~\ref{tab:reasoning-w3} and~\ref{tab:reasoning-w4} (Appendix~\ref{app:nonuniform-reasoning}). At W3, \ourmethod attains the best reasoning accuracy among quantized methods on all three tasks. At W4, methods are closer on MATH-500 and GSM8K, while \ourmethod is strongest on the hardest benchmark (AIME-120).
}

\section{Conclusion and limitations}
\label{sec:conclusion}

We introduce \ourmethod, a principled layer-wise quantization method for post-training quantization of LLMs based on ADMM. By formulating layer-wise quantization as a Hessian-weighted constrained reconstruction problem, \ourmethod jointly optimizes the quantized weights within each layer rather than processing them greedily. We prove convergence guarantees and show that \ourmethod can serve as a drop-in replacement for standard layer-wise quantization solvers in PTQ pipelines that use equivalent transformations such as scaling and rotation. Empirically, \ourmethod substantially improves perplexity and zero-shot accuracy across a broad range of model sizes, from Qwen 1.7B--32B in the weight-only setting and from 1B--8B in our SpinQuant experiments on Qwen and LLaMA, with especially pronounced gains at more aggressive quantization levels.

\noindent\textbf{Limitations.} Although \ourmethod consistently improves quantized model quality, it introduces additional offline quantization cost relative to lightweight baselines such as RTN and GPTQ because it solves an iterative optimization problem for each layer. Similar to other PTQ approaches, \ourmethod relies on calibration activations and a Hessian approximation derived from them, so its performance may depend on how representative the calibration set is for the downstream task. Moreover, we evaluate \ourmethod only on dense autoregressive decoder-only LLMs, not on mixture-of-experts architectures, diffusion language models, or vision transformers. Although the optimization framework is general, these settings have different structure and may require nontrivial adaptation.

\section*{Acknowledgements}
The authors acknowledge research support from the Office of Naval Research Grant N000142512504.

\newpage
\bibliographystyle{plainnat}
\bibliography{ref}
\newpage 
\appendix
\section*{Appendix}\label{app:appendix}
\section{Integration of \ourmethod\ with equivalent transformations}
\label{app:equiv-transform}

We spell out in more detail how equivalent transformations modify the local reconstruction problem solved by \ourmethod. In each case, the transformation rewrites the layer in an equivalent form, after which \ourmethod\ is applied to the transformed weight variable.

\paragraph{Scaling.}
Suppose a per-channel scaling matrix \(\mathbf S\) is applied before quantization. Then the layer can be rewritten as:
\[
\mathbf Y=\mathbf X\mathbf W=(\mathbf X\mathbf S^{-1})(\mathbf S\mathbf W).
\]
Define the scaled weight:
\[
\mathbf U=\mathbf S\mathbf W,
\qquad
\widehat{\mathbf U}=\mathbf S\widehat{\mathbf W}.
\]
Then \ourmethod\ is applied to the transformed local reconstruction problem
\[
\min_{\mathbf U}\; \frac12\mathrm{Tr}\left((\mathbf U-\widehat{\mathbf U})^\top
\mathbf H_{\mathrm{sc}}
(\mathbf U-\widehat{\mathbf U})\right)
\qquad
\text{s.t.}\qquad
\mathbf U\in\mathcal Q
\]
where $
\mathbf H_{\mathrm{sc}}
=
\mathbf X_{\mathrm{sc}}^\top\mathbf X_{\mathrm{sc}}
=
(\mathbf X\mathbf S^{-1})^\top(\mathbf X\mathbf S^{-1})
=
(\mathbf S^{-1})^\top \mathbf H \mathbf S^{-1}$. Similar to \cite{xiao2023smoothquant}, we directly store the scaled quantized weights.

\paragraph{Rotation.}
Suppose rotation matrices \(\mathbf R_1,\mathbf R_2\) are applied before quantization. Then the layer can be rewritten as:
\[
\mathbf Y=\mathbf X\mathbf W
=
(\mathbf X\mathbf R_1)(\mathbf R_1^{-1}\mathbf W\mathbf R_2)\mathbf R_2^{-1}
\]
Define the rotated weight:
\[
\mathbf U=\mathbf R_1^{-1}\mathbf W\mathbf R_2,
\qquad
\widehat{\mathbf U}=\mathbf R_1^{-1}\widehat{\mathbf W}\mathbf R_2
\]
Then \ourmethod\ is applied to the transformed local reconstruction problem:
\[
\min_{\mathbf U}\; \frac12\mathrm{Tr}\left((\mathbf U-\widehat{\mathbf U})^\top
\mathbf H_{\mathrm{rot}}
(\mathbf U-\widehat{\mathbf U})\right)
\qquad
\text{s.t.}\qquad
\mathbf U\in\mathcal Q
\]
where $\mathbf H_{\mathrm{rot}}
=
\mathbf X_{\mathrm{rot}}^\top\mathbf X_{\mathrm{rot}}
=
(\mathbf X\mathbf R_1)^\top(\mathbf X\mathbf R_1)
=
\mathbf R_1^\top\mathbf H\mathbf R_1$. As in other work \cite{liu2024spinquant}, we directly store the rotated quantized weights.

\section{Additional Details on Grid Refresh}
\label{app:grid-refresh}

Here we give more detail on the grid-refresh step used in \ourmethod. The main point is that the quantization grid is not part of the ADMM state in the same way as \(\mathbf W\), \(\mathbf D\), and \(\mathbf V\); rather, it determines how the projection in the \(\mathbf D\)-update is carried out. If the grid is chosen once at initialization and then held fixed, it can become misaligned with the continuous trajectory produced by ADMM.

\paragraph{Stale fixed grid.}
Let \(\theta=(\Delta,z_0)\) denote the scale and zero-point that determine the quantization grid \(\mathcal Q(\theta)\). At iteration \(t\), the point projected in the \(\mathbf D\)-update is
\[
\widetilde{\mathbf W}^{(t)}=\mathbf W^{(t+1)}+\mathbf V^{(t)}/\rho_t.
\]
With a fixed grid \(\theta\), the discrete update is
\[
\mathbf D^{(t+1)}
=
\bm{\Sigma}^{-1}\operatorname{Proj}_{\mathcal Q(\theta)}\!\bigl(\bm{\Sigma}\widetilde{\mathbf W}^{(t)}\bigr).
\]
In practice, \(\theta\) is typically initialized by RTN to the dense weights. Early in the ADMM iterations this is a reasonable choice, but as the continuous iterates move, the projection induced by that initial grid may no longer be well matched to the current \(\widetilde{\mathbf W}^{(t)}\). The projected point can therefore be unnecessarily far from the current continuous ADMM target.

\paragraph{Refreshed grid.}
To address this, we equip \ourmethod with the ability to recompute the grid parameters from the current continuous weights using the same scale/zero-point fitting routine used at initialization. This gives refreshed parameters \(\theta_{\mathrm{new}}=(\Delta_{\mathrm{new}}, z_{0,\mathrm{new}})\), and the same target \(\widetilde{\mathbf W}^{(t)}\) is then re-projected onto the refreshed grid:
\[
\mathbf D_{\mathrm{new}}^{(t+1)}
=
\bm{\Sigma}^{-1}\operatorname{Proj}_{\mathcal Q(\theta_{\mathrm{new}})}\!\bigl(\bm{\Sigma}\widetilde{\mathbf W}^{(t)}\bigr).
\]
The role of the refresh is therefore not to change the continuous ADMM iterate itself, but to update the projection map so that the discrete point better matches the current continuous trajectory.

\paragraph{Acceptance and dual consistency.}
We accept a refreshed grid only when the updated quantized weights improves $\mathbf{D}$-update objective. If the current projected point is \(\mathbf D_{\mathrm{old}}^{(t+1)}\) and the refreshed projection is \(\mathbf D_{\mathrm{new}}^{(t+1)}\), the acceptance criteria reads:
\begin{equation}\label{eq:refreshcri}
\left\| \mathbf{D}^{(t+1)}_{\mathrm{new}} - \left(\mathbf{W}^{(t+1)}+ \mathbf{V}^{(t)}/\rho_t\right) \right\|_F^2 \le \left\| \mathbf{D}^{(t+1)}_{\mathrm{old}} - \left(\mathbf{W}^{(t+1)}+ \mathbf{V}^{(t)}/\rho_t\right) \right\|_F^2.
\end{equation}
After accepting the refresh we replace the discrete point by \(\mathbf D_{\mathrm{new}}^{(t+1)}\) and update the dual variable as:
\begin{equation}
    \mathbf V
\leftarrow
\mathbf V
+
\rho_t\bigl(\mathbf D_{\mathrm{old}}^{(t+1)}-\mathbf D_{\mathrm{new}}^{(t+1)}\bigr).
\end{equation}
This preserves consistency of the current ADMM state after the discrete projection has been changed. To avoid repeated interventions in the projection map from affecting the convergence behavior of the base ADMM iterations, we allow at most one accepted grid refresh.

\section{Proof of \cref{thm:admmq}}\label{app:proofadmm}
We first build a key inequality: 
\begin{equation}\label{eq:proof-eq}
\begin{aligned}
\left\| \mathbf{D}^{(t+1)} - \left(\mathbf{W}^{(t+1)}+ \mathbf{V}^{(t)}/\rho_t\right) \right\|_F^2 \le \left\| \mathbf{D}^{(t)} - \left(\mathbf{W}^{(t+1)}+ \mathbf{V}^{(t)}/\rho_t\right) \right\|_F^2.
\end{aligned}
\end{equation}
\begin{proof}
We consider two scenarios.\\ \textbf{Case 1:} The quantization grid is fixed during this iteration. Since the update rule \eqref{eq:projection} rounds each element of $\mathbf{W}^{(t+1)} + \mathbf{V}^{(t)}/\rho_t$ to its nearest point on the quantization grid, it turns out that $\mathbf{D}^{t+1}$ satisfies
    \begin{equation}
        \mathbf{D}^{(t+1)} = \argmin_{\bm{\Sigma}\mathbf{D} \in \mathcal{Q}} \|\mathbf{D} - (\mathbf{W}^{(t+1)} + \mathbf{V}^{(t)}/\rho_t)\|_F^2.
    \end{equation}
    Therefore, for any $\mathbf{D}$ such that $\bm{\Sigma}\mathbf{D} \in \mathcal{Q}$, it holds
    \begin{equation}
\begin{aligned}
\left\| \mathbf{D}^{(t+1)} - \left(\mathbf{W}^{(t+1)}+ \mathbf{V}^{(t)}/\rho_t\right) \right\|_F^2 \le \left\| \mathbf{D} - \left(\mathbf{W}^{(t+1)}+ \mathbf{V}^{(t)}/\rho_t\right) \right\|_F^2.
\end{aligned}
\end{equation}
Since $\bm{\Sigma}\mathbf{D}^{(t)} \in \mathcal{Q}$, inequality \eqref{eq:proof-eq} holds.

\textbf{Case 2:} The quantization grid has been updated during this iteration. According to the algorithm, we can simplify the update rule as: performing $\mathbf{W}$-update as usual; find $\mathbf{D}^{(t+1)}_{\mathrm{new}}$ on the new quantization grid; and finally the $\mathbf{V}$-update $\mathbf{V}^{(t+1)} = \mathbf{V}^{(t)} + \rho_t(\mathbf{W}^{(t+1)} - \mathbf{D}^{(t+1)}_{\mathrm{new}})$. Note that we will only refresh the grid only if it improves $\mathbf{D}$-update objective. Together with the discussion in Case 1, it follows that
\begin{equation}
\begin{aligned}
\left\| \mathbf{D}^{(t+1)}_{\mathrm{new}} - \left(\mathbf{W}^{(t+1)}+ \mathbf{V}^{(t)}/\rho_t\right) \right\|_F^2 &\le \left\| \mathbf{D}^{(t+1)}_{\mathrm{old}} - \left(\mathbf{W}^{(t+1)}+ \mathbf{V}^{(t)}/\rho_t\right) \right\|_F^2 \\ &\le \left\| \mathbf{D}^{(t)} - \left(\mathbf{W}^{(t+1)}+ \mathbf{V}^{(t)}/\rho_t\right) \right\|_F^2.
\end{aligned}
\end{equation}
\end{proof}

The rest of convergence proof builds primarily on the analysis of \citep[Theorem 1]{meng2024alps}, since the update rules for $\mathbf{W}$ and $\mathbf{V}$ remain unchanged. The convergence argument in \citep[Theorem 1]{meng2024alps} transfers to our setting without modification. We conclude that both $\{\mathbf{D}^{(t)}\}_{t=0}^\infty$ and $\{\mathbf{W}^{(t)}\}_{t=0}^\infty$ converge to a shared limit $\mathbf{\bar{D}}$, which completes the proof.

\section{Additional Experimental Details}\label{app:appendix-exp-details}
\paragraph{Computing environments.} All experiments were conducted on a computing cluster. Unless otherwise specified, we utilized an Intel Xeon Gold 6248 machine with 16 CPU cores and a single NVIDIA L40 48GB / A100 80GB / H100 80GB GPU/ H200 144GB GPU. When runtime compression results are reported, all experiments have been run on the same node (including GPU) configuration. All language models and quantization methods were implemented using the PyTorch library (\cite{paszke2017automatic}).

\paragraph{Implementation details of \ourmethod.}
Following \cite{frantar2022gptq}, we add a damping term to the Hessian: $\mathbf{H}' = \mathbf{H}+0.01\cdot\operatorname{Tr}(\mathbf{X}^\top\mathbf{X})\mathbf{I}$.

We use $T=300$ maximum number of ADMM iterations with initial penalty $\rho_0 = 0.1$ and $\gamma = 1.1$. 

The MSE grid search for clipping ratios uses 100 grid points with maximum shrink ratio 0.8. For integration with SpinQuant, \ourmethod replaces only the GPTQ weight quantization step; all rotation optimization, activation quantization, and KV cache quantization remain unchanged from the SpinQuant pipeline.

\paragraph{Algorithm Runtime and Memory Results.}
Table~\ref{tab:qwen3_runtime_memory} reports wall-clock quantization time and peak GPU memory usage for GPTQ and ADMM-Q on Qwen3 weight-only quantization. While ADMM-Q consistently delivers better quantized model quality in our experiments, it does so by replacing GPTQ's greedy one-pass updates with an iterative optimization procedure, so some additional offline cost is expected. The table shows that the runtime overhead grows with model size, but remains practically manageable: even for Qwen3-32B, quantization finishes in under two hours for both methods. Peak memory, reported as a ratio normalized so that \ourmethod\ equals 1.0, shows that GPTQ uses roughly 94\% of \ourmethod's peak across all model sizes, a consistent but modest difference. Since weight quantization is a one-time preprocessing step rather than part of inference-time deployment, we view both the increase in runtime and the slightly higher peak memory as reasonable tradeoffs for the consistent gains in perplexity and zero-shot accuracy.

\begin{table}[H]
\centering
\footnotesize
\setlength{\tabcolsep}{4.0pt}
\renewcommand{\arraystretch}{1.05}
\resizebox{\textwidth}{!}{%
\begin{tabular}{l|cc|cc|c|cc|c|cc}
\toprule
\textbf{Model}
& \multicolumn{2}{c|}{\textbf{Runtime (min)}}
& \multicolumn{2}{c|}{\textbf{Peak quant. memory}}
& \multicolumn{3}{c|}{\textbf{Peak inference memory}} 
& \multicolumn{3}{c}{\textbf{Storage memory}} \\
\cmidrule(lr){2-3} \cmidrule(lr){4-5} \cmidrule(lr){6-8} \cmidrule(lr){9-11}
& \texttt{GPTQ} & \cellcolor{darkgreen!10}\ourmethod
& \texttt{GPTQ} & \cellcolor{darkgreen!10}\ourmethod
& \cellcolor{gray!10}\textbf{Dense} & \texttt{GPTQ} & \cellcolor{darkgreen!10}\ourmethod 
& \cellcolor{gray!10}\textbf{Dense} & \texttt{GPTQ} & \cellcolor{darkgreen!10}\ourmethod \\
\midrule
1.7B-Base & \textbf{1.95}  & \cellcolor{darkgreen!10} 2.22   & \textbf{0.956 $\pm$ 0.012} & \cellcolor{darkgreen!10} 1.000 & \cellcolor{gray!10} 1.000 & \textbf{0.285 $\pm$ 0.013} & \cellcolor{darkgreen!10} 0.294 $\pm$ 0.024 & \cellcolor{gray!10} 1.000 & \textbf{0.255} & \cellcolor{darkgreen!10} \textbf{0.255} \\
4B-Base   & \textbf{4.44}  & \cellcolor{darkgreen!10} 7.20   & \textbf{0.943 $\pm$ 0.014} & \cellcolor{darkgreen!10} 1.000 & \cellcolor{gray!10} 1.000 & \textbf{0.281 $\pm$ 0.015} & \cellcolor{darkgreen!10} 0.285 $\pm$ 0.018 & \cellcolor{gray!10} 1.000 & \textbf{0.254} & \cellcolor{darkgreen!10} 0.255 \\
8B-Base   & \textbf{7.31}  & \cellcolor{darkgreen!10} 13.91  & \textbf{0.938 $\pm$ 0.013} & \cellcolor{darkgreen!10} 1.000 & \cellcolor{gray!10} 1.000 & \textbf{0.279 $\pm$ 0.017} & \cellcolor{darkgreen!10} 0.282 $\pm$ 0.012 & \cellcolor{gray!10} 1.000 & \textbf{0.253} & \cellcolor{darkgreen!10} 0.254 \\
14B-Base  & \textbf{13.27} & \cellcolor{darkgreen!10} 33.79  & \textbf{0.944 $\pm$ 0.012} & \cellcolor{darkgreen!10} 1.000 & \cellcolor{gray!10} 1.000 & \textbf{0.284 $\pm$ 0.011} & \cellcolor{darkgreen!10} \textbf{0.284 $\pm$ 0.012} & \cellcolor{gray!10} 1.000 & \textbf{0.253} & \cellcolor{darkgreen!10} \textbf{0.253} \\
32B       & \textbf{33.97} & \cellcolor{darkgreen!10} 117.73 & \textbf{0.941 $\pm$ 0.013} & \cellcolor{darkgreen!10} 1.000 & \cellcolor{gray!10} 1.000 & \textbf{0.273 $\pm$ 0.010} & \cellcolor{darkgreen!10} \textbf{0.273 $\pm$ 0.008} & \cellcolor{gray!10} 1.000 & \textbf{0.252} & \cellcolor{darkgreen!10} 0.253 \\
\bottomrule
\end{tabular}
}
\caption{Wall-clock quantization time (minutes) and memory usage metrics. Peak quantization memory compares GPTQ and \ourmethod, normalized so \ourmethod is 1; values below 1 indicate lower peak memory usage than \ourmethod. Peak inference memory and storage memory, with Dense normalized to 1, indicate that both GPTQ and \ourmethod use approximately 28\% of the memory during inference and take up just over 25\% as much storage memory.}
\label{tab:qwen3_runtime_memory}
\end{table}

\paragraph{Ablation Study on \ourmethod Algorithm Components.}

Table~\ref{tab:admmq-ablation-qwen3-8b-w3w4} shows that the performance gains of \ourmethod\ come from the combination of its main components. At W4, the full method improves over \texttt{GPTQ} on all three perplexity metrics, reducing Wiki-2 from 7.82 to 7.56, C4 from 13.39 to 13.01, and PTB from 19.82 to 19.04, while remaining competitive in zero-shot accuracy. At W3, the same pattern is more pronounced: \ourmethod\ substantially improves over \texttt{GPTQ}, reducing Wiki-2 from 12.85 to 10.06, C4 from 21.27 to 18.25, and PTB from 35.01 to 27.30, while also improving zero-shot accuracy from 54.43 to 57.80. Among the individual ablations, removing diagonal preconditioning causes the largest loss, especially at 3 bits, indicating that conditioning the layer-wise subproblem is critical for effective ADMM updates. Removing adaptive $\rho$ also leads to a clear drop in quality, while local search provides a smaller but still meaningful gain overall. When all three components are disabled, performance deteriorates sharply, particularly at W3.

\begin{table}[H]
\centering
\smaller
\setlength{\tabcolsep}{3.6pt}
\renewcommand{\arraystretch}{1.12}
\newcolumntype{C}[1]{>{\centering\arraybackslash}p{#1}}
\begin{tabular}{l|C{3.0em}C{3.0em}C{3.0em}C{3.0em}|C{3.0em}C{3.0em}C{3.0em}C{3.0em}}
\toprule
& \multicolumn{4}{c|}{\cellcolor{blue!8}\textbf{W4}}
& \multicolumn{4}{c}{\cellcolor{blue!8}\textbf{W3}} \\
\textbf{Method}
& \rotatebox{50}{\textbf{Wiki-2}}
& \rotatebox{50}{\textbf{C4}}
& \rotatebox{50}{\textbf{PTB}}
& \rotatebox{50}{\textbf{0-shot}}
& \rotatebox{50}{\textbf{Wiki-2}}
& \rotatebox{50}{\textbf{C4}}
& \rotatebox{50}{\textbf{PTB}}
& \rotatebox{50}{\textbf{0-shot}} \\
\midrule
\rowcolor{gray!10}
Dense
& 7.00 & 11.91 & 17.12 & 69.07
& 7.00 & 11.91 & 17.12 & 69.07 \\
\texttt{GPTQ}
& 7.82 & 13.39 & 19.82 & \textbf{68.93}
& 12.85 & 21.27 & 35.01 & 54.43 \\
\midrule
\rowcolor{darkgreen!10}
\ourmethod\ (full)
& \textbf{7.56} & \textbf{13.01} & \textbf{19.04} & 68.14
& 10.06 & \textbf{18.25} & \textbf{27.30} & \textbf{57.80} \\
\quad w/o local search
& 7.58 & 13.11 & 19.22 & 68.41
& \textbf{9.93} & 18.44 & 28.95 & 57.36 \\
\quad w/o adaptive $\rho$
& 7.73 & 13.24 & 19.46 & 67.63
& 10.97 & 19.86 & 30.81 & 55.23 \\
\quad w/o diag.\ precond.
& 8.03 & 13.66 & 20.59 & 66.95
& 14.68 & 26.42 & 45.74 & 50.83 \\
\quad all off
& 9.64 & 16.34 & 26.52 & 65.79
& 42.47 & 57.76 & 115.86 & 41.94 \\
\bottomrule
\end{tabular}
\vspace{3pt}
\caption{Ablation study for \ourmethod\ on Qwen3-8B-Base under symmetric per-channel weight-only quantization at 4 bits (W4) and 3 bits (W3). PPL columns ($\downarrow$) report perplexity on WikiText-2, C4, and Penn Treebank; 0-shot ($\uparrow$) is average zero-shot accuracy. Dense and \texttt{GPTQ} are shown for reference. The same qualitative pattern holds at both precisions, but the impact of removing key components is substantially larger at 3 bits, particularly for diagonal preconditioning.}
\label{tab:admmq-ablation-qwen3-8b-w3w4}
\end{table}

\paragraph{Ablation Study on \ourmethod with SpinQuant.}

Table~\ref{tab:admmq-ablation-spinquant-w3w4} extends the ablation study to the weight-and-activation setting (W$b$A4KV4) with SpinQuant on LLaMA-3 8B and LLaMA-3.2 1B-Instruct. In this more constrained setting, local search and adaptive $\rho$ are the most impactful components, each causing substantial degradation when removed, particularly at W3 where removing either one roughly doubles the perplexity gap to the dense baseline on LLaMA-3 8B. Diagonal preconditioning provides a smaller but steady improvement across all configurations. The ``all off'' configuration degrades catastrophically at W3, reaching perplexities above 150 on the smaller model, confirming that all three components are essential for stable low-bit quantization in the weight-and-activation setting.

\begin{table}[H]
\centering
\smaller
\setlength{\tabcolsep}{2.8pt}
\renewcommand{\arraystretch}{1.12}
\newcolumntype{C}[1]{>{\centering\arraybackslash}p{#1}}
\resizebox{\textwidth}{!}{%
\begin{tabular}{l|C{2.6em}C{2.6em}C{2.6em}C{2.6em}|C{2.6em}C{2.6em}C{2.6em}C{2.6em}|C{2.6em}C{2.6em}C{2.6em}C{2.6em}|C{2.6em}C{2.6em}C{2.6em}C{2.6em}}
\toprule
& \multicolumn{4}{c|}{\cellcolor{blue!8}\textbf{LLaMA-3 8B — W4A4KV4}}
& \multicolumn{4}{c|}{\cellcolor{blue!8}\textbf{LLaMA-3 8B — W3A4KV4}}
& \multicolumn{4}{c|}{\cellcolor{blue!8}\textbf{1B-Inst.\ — W4A4KV4}}
& \multicolumn{4}{c}{\cellcolor{blue!8}\textbf{1B-Inst.\ — W3A4KV4}} \\
\textbf{Method}
& \rotatebox{50}{\textbf{Wiki-2}}
& \rotatebox{50}{\textbf{C4}}
& \rotatebox{50}{\textbf{PTB}}
& \rotatebox{50}{\textbf{0-shot}}
& \rotatebox{50}{\textbf{Wiki-2}}
& \rotatebox{50}{\textbf{C4}}
& \rotatebox{50}{\textbf{PTB}}
& \rotatebox{50}{\textbf{0-shot}}
& \rotatebox{50}{\textbf{Wiki-2}}
& \rotatebox{50}{\textbf{C4}}
& \rotatebox{50}{\textbf{PTB}}
& \rotatebox{50}{\textbf{0-shot}}
& \rotatebox{50}{\textbf{Wiki-2}}
& \rotatebox{50}{\textbf{C4}}
& \rotatebox{50}{\textbf{PTB}}
& \rotatebox{50}{\textbf{0-shot}} \\
\midrule
\rowcolor{gray!10}
Dense
& 6.14 & 9.44 & 14.84 & 63.94
& 6.14 & 9.44 & 14.84 & 63.94
& 13.16 & 21.30 & 37.31 & 53.22
& 13.16 & 21.30 & 37.31 & 53.22 \\
\texttt{GPTQ}
& 7.34 & 12.16 & 17.94 & \textbf{60.55}
& 8.50 & 15.01 & 21.27 & 56.51
& 15.81 & 26.11 & 48.69 & 48.27
& 19.70 & 34.07 & 65.22 & 45.76 \\
\midrule
\rowcolor{darkgreen!10}
\ourmethod\ (full)
& \textbf{7.26} & \textbf{12.08} & \textbf{17.74} & 60.52
& \textbf{8.02} & \textbf{14.29} & \textbf{20.13} & \textbf{57.59}
& \textbf{15.55} & \textbf{25.80} & \textbf{47.76} & \textbf{49.31}
& \textbf{17.88} & \textbf{31.16} & \textbf{56.65} & \textbf{46.84} \\
\quad w/o local search
& 7.97 & 13.09 & 18.93 & 58.32
& 10.93 & 20.91 & 29.32 & 51.06
& 19.45 & 31.18 & 58.28 & 46.20
& 28.12 & 48.41 & 100.83 & 42.74 \\
\quad w/o adaptive $\rho$
& 7.94 & 13.16 & 19.17 & 57.96
& 11.52 & 22.00 & 31.97 & 51.78
& 17.45 & 28.28 & 53.96 & 48.39
& 28.32 & 49.34 & 91.35 & 42.85 \\
\quad w/o diag.\ precond.
& 7.65 & 12.72 & 18.59 & 58.43
& 9.44 & 17.39 & 25.30 & 54.44
& 16.03 & 26.52 & 48.72 & 48.79
& 19.76 & 35.48 & 66.57 & 45.17 \\
\quad all off
& 8.89 & 14.43 & 20.98 & 57.12
& 56.35 & 82.98 & 174.21 & 40.46
& 21.87 & 33.58 & 65.31 & 45.71
& 150.16 & 184.86 & 396.32 & 37.80 \\
\bottomrule
\end{tabular}
}
\vspace{3pt}
\caption{Ablation study for \ourmethod\ within the SpinQuant pipeline on LLaMA-3 8B and LLaMA-3.2 1B-Instruct under symmetric per-channel W$b$A4KV4 quantization at 4 bits (W4) and 3 bits (W3). PPL columns ($\downarrow$) report perplexity on WikiText-2, C4, and Penn Treebank; 0-shot ($\uparrow$) is average zero-shot accuracy. Dense and \texttt{GPTQ} baselines are shown for reference.}
\label{tab:admmq-ablation-spinquant-w3w4}
\end{table}

\section{Other quantization formats and reasoning evaluations}
\label{app:nonuniform-reasoning}

\crnew{
\paragraph{ICQuant integration.}
Tables~\ref{tab:icquant-1p7b} and~\ref{tab:icquant-4b} report WikiText-2 / C4 perplexity when \ourmethod is combined with the ICQuant dual-grid format. RTN-IC and GPTQ-IC use the same outlier set; \ourmethod-IC differs only in the layer-wise solver.
}

\begin{table}[H]
\centering
\small
\begin{tabular}{llllll}
\toprule
\crnew{Bits} & \crnew{RTN-IC} & \crnew{GPTQ} & \crnew{GPTQ-IC} & \crnew{\ourmethod} & \crnew{\ourmethod-IC} \\
\midrule
\crnew{W2} & \crnew{2.8e8 / 1.7e8} & \crnew{527.19 / 1357.12} & \crnew{51.94 / 147.08} & \crnew{39.49 / 255.75} & \crnew{\textbf{18.40 / 50.88}} \\
\crnew{W3} & \crnew{16.58 / 25.65} & \crnew{23.54 / 51.02} & \crnew{11.99 / 20.09} & \crnew{16.79 / 38.58} & \crnew{\textbf{10.36 / 17.46}} \\
\crnew{W4} & \crnew{10.29 / 16.91} & \crnew{12.53 / 20.99} & \crnew{9.91 / 16.30} & \crnew{11.19 / 18.98} & \crnew{\textbf{9.62 / 15.90}} \\
\bottomrule
\end{tabular}
\caption{\crnew{Qwen3-1.7B-Base (FP16: 9.40 / 15.42) with ICQuant. Entries are WikiText-2 / C4.}}
\label{tab:icquant-1p7b}
\end{table}

\begin{table}[H]
\centering
\small
\begin{tabular}{llllll}
\toprule
\crnew{Bits} & \crnew{RTN-IC} & \crnew{GPTQ} & \crnew{GPTQ-IC} & \crnew{\ourmethod} & \crnew{\ourmethod-IC} \\
\midrule
\crnew{W2} & \crnew{4.2e7 / 3.3e7} & \crnew{33.60 / 95.05} & \crnew{22.78 / 43.91} & \crnew{23.90 / 73.17} & \crnew{\textbf{12.97 / 29.42}} \\
\crnew{W3} & \crnew{9.81 / 16.12} & \crnew{13.26 / 24.27} & \crnew{8.49 / 14.59} & \crnew{13.02 / 25.32} & \crnew{\textbf{8.42 / 14.51}} \\
\crnew{W4} & \crnew{8.44 / 14.00} & \crnew{8.52 / 14.59} & \crnew{8.05 / 13.61} & \crnew{8.48 / 14.58} & \crnew{\textbf{8.01 / 13.58}} \\
\bottomrule
\end{tabular}
\caption{\crnew{Qwen3-4B-Base (FP16: 7.90 / 13.31) with ICQuant. Entries are WikiText-2 / C4.}}
\label{tab:icquant-4b}
\end{table}

\crnew{
\paragraph{SqueezeLLM integration.}
Tables~\ref{tab:squeezellm-w3} and~\ref{tab:squeezellm-w2} report results when \ourmethod is combined with the SqueezeLLM non-uniform codebook and dense-and-sparse mask. Both methods use the same codebook and sparse mask; they differ only in how dense entries are quantized.
}

\begin{table}[H]
\centering
\small
\begin{tabular}{llll}
\toprule
\crnew{Model} & \crnew{FP16} & \crnew{SqueezeLLM} & \crnew{\ourmethod+SqueezeLLM} \\
\midrule
\crnew{1.7B} & \crnew{9.40 / 15.42 / 62.54} & \crnew{11.58 / 18.75 / 57.75} & \crnew{\textbf{10.32 / 17.51 / 58.13}} \\
\crnew{4B} & \crnew{7.90 / 13.31 / 66.72} & \crnew{9.13 / 15.05 / 63.64} & \crnew{\textbf{8.41 / 14.50 / 65.27}} \\
\crnew{8B} & \crnew{7.00 / 11.91 / 69.07} & \crnew{7.78 / 13.10 / 67.52} & \crnew{\textbf{7.39 / 12.81 / 67.69}} \\
\crnew{14B} & \crnew{6.38 / 10.99 / 71.75} & \crnew{6.96 / 11.84 / 70.38} & \crnew{\textbf{6.70 / 11.73 / 70.87}} \\
\crnew{32B} & \crnew{7.61 / 12.45 / 71.49} & \crnew{8.28 / 13.39 / 71.33} & \crnew{\textbf{7.94 / 13.20 / 71.57}} \\
\bottomrule
\end{tabular}
\caption{\crnew{Qwen3 W3 non-uniform quantization with SqueezeLLM. Entries are WikiText-2 / C4 / zero-shot accuracy.}}
\label{tab:squeezellm-w3}
\end{table}

\begin{table}[H]
\centering
\small
\begin{tabular}{llll}
\toprule
\crnew{Model} & \crnew{FP16} & \crnew{SqueezeLLM} & \crnew{\ourmethod+SqueezeLLM} \\
\midrule
\crnew{1.7B} & \crnew{9.40 / 15.42 / 62.54} & \crnew{154.13 / 240.10 / 42.02} & \crnew{\textbf{17.28 / 47.34 / 46.49}} \\
\crnew{4B} & \crnew{7.90 / 13.31 / 66.72} & \crnew{22.81 / 35.60 / 51.33} & \crnew{\textbf{11.80 / 26.06 / 55.02}} \\
\crnew{8B} & \crnew{7.00 / 11.91 / 69.07} & \crnew{15.07 / 24.94 / 57.07} & \crnew{\textbf{9.66 / 19.16 / 58.55}} \\
\crnew{14B} & \crnew{6.38 / 10.99 / 71.75} & \crnew{10.73 / 17.78 / 60.91} & \crnew{\textbf{8.37 / 16.21 / 62.83}} \\
\crnew{32B} & \crnew{7.61 / 12.45 / 71.49} & \crnew{12.82 / 19.94 / 66.51} & \crnew{\textbf{10.07 / 17.64 / 68.25}} \\
\bottomrule
\end{tabular}
\caption{\crnew{Qwen3 W2 non-uniform quantization with SqueezeLLM. Entries are WikiText-2 / C4 / zero-shot accuracy.}}
\label{tab:squeezellm-w2}
\end{table}

\crnew{
\paragraph{Reasoning benchmarks.}
Tables~\ref{tab:reasoning-w3} and~\ref{tab:reasoning-w4} report AIME-120, MATH-500, and GSM8K accuracy for Qwen3-8B under asymmetric W3/W4 group-128 weight-only quantization, following~\citet{liu2025quantizationhurtsreasoningempirical}.
}

\begin{table}[H]
\centering
\small
\begin{tabular}{lcccc}
\toprule
\crnew{Dataset} & \crnew{FP16} & \crnew{\ourmethod} & \crnew{GPTQ} & \crnew{AWQ} \\
\midrule
\crnew{AIME-120} & \crnew{69.17 $\pm$ 4.20} & \crnew{\textbf{63.33 $\pm$ 4.37}} & \crnew{57.50 $\pm$ 4.50} & \crnew{43.33 $\pm$ 4.52} \\
\crnew{MATH-500} & \crnew{97.20 $\pm$ 0.74} & \crnew{\textbf{95.60 $\pm$ 0.92}} & \crnew{94.80 $\pm$ 0.99} & \crnew{91.80 $\pm$ 1.23} \\
\crnew{GSM8K} & \crnew{95.15 $\pm$ 0.59} & \crnew{\textbf{94.24 $\pm$ 0.64}} & \crnew{94.16 $\pm$ 0.65} & \crnew{93.56 $\pm$ 0.68} \\
\bottomrule
\end{tabular}
\caption{\crnew{Qwen3-8B reasoning accuracy at W3G128.}}
\label{tab:reasoning-w3}
\end{table}

\begin{table}[H]
\centering
\small
\begin{tabular}{lcccc}
\toprule
\crnew{Dataset} & \crnew{FP16} & \crnew{\ourmethod} & \crnew{GPTQ} & \crnew{AWQ} \\
\midrule
\crnew{AIME-120} & \crnew{69.17 $\pm$ 4.20} & \crnew{\textbf{71.67 $\pm$ 4.11}} & \crnew{64.17 $\pm$ 4.31} & \crnew{68.33 $\pm$ 4.24} \\
\crnew{MATH-500} & \crnew{97.20 $\pm$ 0.74} & \crnew{97.40 $\pm$ 0.71} & \crnew{\textbf{97.60 $\pm$ 0.68}} & \crnew{96.80 $\pm$ 0.79} \\
\crnew{GSM8K} & \crnew{95.15 $\pm$ 0.59} & \crnew{94.77 $\pm$ 0.61} & \crnew{95.07 $\pm$ 0.60} & \crnew{\textbf{95.68 $\pm$ 0.56}} \\
\bottomrule
\end{tabular}
\caption{\crnew{Qwen3-8B reasoning accuracy at W4G128.}}
\label{tab:reasoning-w4}
\end{table}


\end{document}